\documentclass[11pt]{article}

\usepackage{tikz}
\usepackage{multicol}
\usetikzlibrary{positioning}
\usetikzlibrary{decorations.pathreplacing}

\usepackage[utf8]{inputenc}
\usepackage[colorlinks = true,
            linkcolor = blue,
            urlcolor  = blue,
            citecolor = blue,
            anchorcolor = blue]{hyperref}
\usepackage{amsmath}
\usepackage{amsthm}
\usepackage{amssymb}
\usepackage{xcolor}
\usepackage{algorithm,caption}
\usepackage{cleveref}
\usepackage{graphicx}
\usepackage{scalerel}
\usepackage{enumitem}
\usepackage{natbib}
\usepackage{framed}

\usepackage{algpseudocode}

\definecolor{gray}{gray}{0.4}

\usepackage{geometry}
\makeatletter
\theoremstyle{plain}
\newtheorem{theorem}{\protect\theoremname}
\newtheorem*{theorem*}{\protect\theoremname}
\newtheorem*{prop*}{\protect\theoremname}
\newtheorem{definition}{\protect\definitionname}
\theoremstyle{definition}

\newtheorem{example}[definition]{\protect\examplename}
\theoremstyle{plain}
\newtheorem{lem}[definition]{\protect\lemmaname}

\newtheorem*{cor*}{\protect\corollaryname}

\newtheorem{prop}[definition]{\protect\propname}

\newtheorem*{question*}{\protect\questionname}
\newtheorem*{assumption*}{\protect\assumptionname}

\makeatother

\newtheorem{innercustomthm}{Theorem}
\newenvironment{customthm}[1]
  {\renewcommand\theinnercustomthm{#1}\innercustomthm}
  {\endinnercustomthm}
\newtheorem{innercustomprop}{Proposition}
\newenvironment{customprop}[1]
  {\renewcommand\theinnercustomprop{#1}\innercustomprop}
  {\endinnercustomprop}
\newtheorem{innercustomdef}{Definition}
\newenvironment{customdef}[1]
  {\renewcommand\theinnercustomdef{#1}\innercustomdef}
  {\endinnercustomdef}

\newtheorem{innercustomexample}{Example}
\newenvironment{customexample}[1]
  {\renewcommand\theinnercustomexample{#1}\innercustomexample}
  {\endinnercustomdef}

\providecommand{\questionname}{Question}
\providecommand{\assumptionname}{Assumption}
\providecommand{\observationname}{Observation}
\providecommand{\corollaryname}{Corollary}
\providecommand{\definitionname}{Definition}
\providecommand{\lemmaname}{Lemma}

\providecommand{\theoremname}{Theorem}
\providecommand{\exercisename}{Exercise}
\providecommand{\examplename}{Example}
\providecommand{\remarkname}{Remark}
\providecommand{\factname}{Fact}
\providecommand{\propname}{Proposition}

\newcommand{\case}[4]{ \begin{cases}
    #1\quad &#2\\
    #3\quad &#4\end{cases}}
    \newcommand{\inProd}[2]{\langle #1,#2\rangle}

\newcommand{\OPT}{\mathtt{OPT}}
\newcommand{\inp}{\mathtt{in}}
\newcommand{\outp}{\mathtt{out}}
\newcommand{\TV}{\mathtt{TV}}
\newcommand{\vc}{\mathtt{VC}}
\newcommand{\R}{\mathbb{R}}
\renewcommand{\S}{\mathbb{S}}

\newcommand{\eps}{\epsilon}

\newcommand{\sign}{\mathsf{sign}}
\newcommand{\maj}{\mathsf{maj}}

\newcommand{\N}{\mathbb{N}}
\newcommand{\NN}{\mathcal{N}}
\newcommand{\Ex}{\mathop{\mathbb{E}}}
\renewcommand{\Pr}{\mathop{\mathbb{P}}}
\newcommand{\T}{\mathcal{T}}
\newcommand{\ol}[1]{\overline{#1}}

\newcommand{\dc}{\mathsf{dc}}
\newcommand{\sr}{\mathsf{sr}}

        \hypersetup{
           breaklinks=true,   
           colorlinks=true,   
        }

  \author{Bogdan Chornomaz \thanks{Departments of Mathematics, Technion.} \and  \and Shay Moran\footnote{Departments of Mathematics, Computer Science, and Data and Decision Sciences, Technion and Google Research.
 Robert J.\ Shillman Fellow; supported by ISF grant 1225/20, by BSF grant 2018385, by Israel PBC-VATAT, by the Technion Center for Machine Learning and Intelligent Systems (MLIS), and by the European Union (ERC, GENERALIZATION, 101039692). Views and opinions expressed are however those of the author(s) only and do not necessarily reflect those of the European Union or the European Research Council Executive Agency. Neither the European Union nor the granting authority can be held responsible for them. 
 } \and Tom Waknine \footnotemark[1]}

\title{On Reductions and Representations of Learning Problems in Euclidean Spaces}
\begin{document}
\maketitle

\begin{abstract}
Many practical prediction algorithms represent inputs in Euclidean space and replace the discrete 0/1 classification loss with a real-valued surrogate loss, effectively reducing classification tasks to stochastic optimization. In this paper, we investigate the expressivity of such reductions in terms of key resources, including dimension and the role of randomness.

We establish bounds on the minimum Euclidean dimension $D$ needed to reduce a concept class with VC dimension $d$ to a Stochastic Convex Optimization (SCO) problem in $\mathbb{R}^D$, formally addressing the intuitive interpretation of the VC dimension as the number of parameters needed to learn the class. To achieve this, we develop a generalization of the Borsuk-Ulam Theorem that combines the classical topological approach with convexity considerations. Perhaps surprisingly, we show that, in some cases, the number of parameters $D$ must be exponentially larger than the VC dimension $d$, even if the reduction is only slightly non-trivial. We also present natural classification tasks that can be represented in much smaller dimensions by leveraging randomness, as seen in techniques like random initialization. This result resolves an open question posed by Kamath, Montasser, and Srebro (COLT 2020). 

Our findings introduce new variants of \emph{dimension complexity} (also known as \emph{sign-rank}), a well-studied parameter in learning and complexity theory. Specifically, we define an approximate version of sign-rank and another variant that captures the minimum dimension required for a reduction to SCO. We also propose several open questions and directions for future research.
\end{abstract}

\section{Introduction}
Reduction is a fundamental concept in computer science, serving as a basic primitive in both computability and complexity theory. It plays a crucial role in defining complexity classes, such as P and NP, by enabling structured transformations of problems into one another. Through reductions, we can compare the difficulty of different problems, determine their relative complexity, and classify them accordingly. This framework has been instrumental in understanding computational limits and identifying problems that are efficiently solvable or intractable, shaping the study of algorithms and theoretical computer science.

Reductions are also a common and powerful tool in machine learning. Implicitly, reductions are used whenever one solves a problem by translating it into an already solved one. For instance, any linear representation of a classification task, such as those used in kernel machines, can be viewed as a reduction from the original learning task to linear classification (half-spaces). Another prominent example is the use of convex surrogate losses in place of the discrete 0/1 classification loss, which reduces the classification task to (stochastic) convex optimization. Similarly, representation learning, which focuses on learning features, and meta-learning, which focuses on learning which learning algorithm to use, can both be interpreted as forms of learning reductions. This perspective is evident, for example, in the formulation of meta-learning by \citet*{AliakbarpourB0S24}.

Historically, reductions in Valiant’s PAC learning model have been studied since the 1980s, primarily in the form of representations—mapping a concept class we wish to learn into a concept class we know how to learn~\citep*{PittW90}. In particular, geometric representations as half-spaces have been thoroughly explored in both learning theory and complexity theory, a partial list includes~\citep*{Ben-DavidES02,LinialS08,ForsterKLMSS01,ForsterSS01, Forster01,Alon17,KamathMS20,HatamiHM22}.

In this paper, we extend this line of research by studying reductions to arbitrary stochastic convex optimization (SCO) problems~\citep*{Shalev-ShwartzSSS09}. We relax the notion of dimension complexity by examining the minimum dimension $d$ for which a given class can be reduced to a $d$-dimensional SCO problem. Additionally, we explore reductions that exploit randomness and agnostic learners, demonstrating their advantages over naive representation-based reductions.

\paragraph{Organization.} We begin in \Cref{sec:results} by presenting our main results. In \Cref{sec:proofoverview}, we provide an overview of the central tools and ideas that underpin our proofs. Next, \Cref{sec:examples} presents natural examples of reductions that illustrate the tightness of our results. In \Cref{sec:prelim}, we cover preliminaries and basic definitions. The remaining sections are devoted to the detailed proofs.

\section{Main results}\label{sec:results}
In this section, we assume familiarity with basic concepts in learning theory such as concept classes, loss functions, and VC dimension. These definitions are provided in detail in Section~\ref{sec:prelim}.

We focus on reducing realizable-case classification problems to well-studied geometric learning tasks: (i) stochastic convex optimization (SCO) and (ii) linear classification. Our impossibility results extend even to non-uniform (distribution-dependent) learning, where sample complexity may vary depending on the input distribution. We examine how the dimension of the reduced problem relates to the VC dimension of the original classification task, and explore whether introducing randomness can help reduce this dimension.

Both the VC dimension and the Euclidean dimension are key parameters in classification and stochastic convex optimization (SCO), respectively: the VC dimension is the fundamental parameter characterizing PAC learnability, as highlighted by the fundamental theorem of PAC learning~\citep{shays14}. In SCO, the Euclidean dimension is closely tied to computational and space complexity; for example, the complexity of arithmetic operations—essential for computing gradients (at training time) and predictions (at test time)—scales with this dimension. It also influences model interpretability, as models with fewer parameters are generally easier for humans to understand.
Studying the relationship between these dimensions also addresses the intuitive interpretation of the VC dimension as the number of parameters needed to describe a concept class; we discuss this in further detail after Theorem~\ref{t:sco}.

A natural approach to relate the VC dimension and the Euclidean dimension would be through a sample complexity analysis, such as showing that if a binary concept class $C$ is reducible to an SCO problem in $\mathbb{R}^d$, then its sample complexity is at most $O(d)$. However, this approach cannot work because, in SCO, the Euclidean dimension does not determine sample complexity; instead, it depends on factors like Lipschitzness and the diameter of the parameter space. To overcome this, we employ a topological argument, specifically a variant of the Borsuk-Ulam Theorem.

\paragraph{Section Organization.}
In Section \ref{sec:reductions}, we formally define reductions between learning tasks. In Section~\ref{sec:sco}, we present our main result on reductions from classification to stochastic convex optimization, and in Section~\ref{sec:sign-rank}, we present our results on representing classification tasks using half-spaces, a well-studied and useful special case of reductions.

\subsection{Reductions}\label{sec:reductions}

\begin{figure}[hbt]
    \centering
	    \begin{tikzpicture}[
        box/.style={draw, rectangle, minimum width=4cm, minimum height=1.2cm, align=center}, 
        arrow/.style={->, thick},
        decorate,
        decoration={brace, amplitude=8pt} 
    ]


    \node (A_instances) [box] {Instances of $A$};
    \node (A_solutions) [box, below=1.5cm of A_instances] {Solutions of $A$};

    \node (B_instances) [box, right=7cm of A_instances] {Instances of $B$};
    \node (B_solutions) [box, below=1.5cm of B_instances] {Solutions of $B$};

    \draw[arrow] (A_instances.east) -- (B_instances.west) node[midway, above] {\small instances of $A$ $\to$ instances of $B$};
    \draw[arrow] (B_solutions.west) -- (A_solutions.east) node[midway, below] {\small solutions of $A$ $\leftarrow$ solutions of $B$};
    
    \draw[arrow] (B_instances.south) -- (B_solutions.north) node[midway, right] {\small Algorithm for $B$};

    \draw[decorate,decoration={brace,amplitude=8pt}] ([xshift=1.0cm, yshift=0.5cm]A_instances.east) -- ([xshift=-1.0cm, yshift=0.5cm]B_instances.west) node[midway, above=0.3cm] {\small $r_{\mathtt{input}}$};
    \draw[decorate,decoration={brace,amplitude=8pt}] ([xshift=-1.0cm, yshift=-0.5cm]B_solutions.west) -- ([xshift=1.0cm, yshift=-0.5cm]A_solutions.east) node[midway, below=0.3cm] {\small $r_{\mathtt{output}}$};

    \node at (A_instances.north) [above=0.5cm] {\small Problem $A$};
    \node at (B_instances.north) [above=0.5cm] {\small Problem $B$};

    \end{tikzpicture}
    \caption{A reduction from problem $A$, which we wish to solve, to problem $B$, which we can solve. The reduction maps instances of $A$ to instances of $B$, and solutions of $B$ back to solutions of $A$. The reduction is successful if, when combined with an algorithm for $B$, it solves problem~$A$.}
    \label{fig:reduction}
\end{figure}
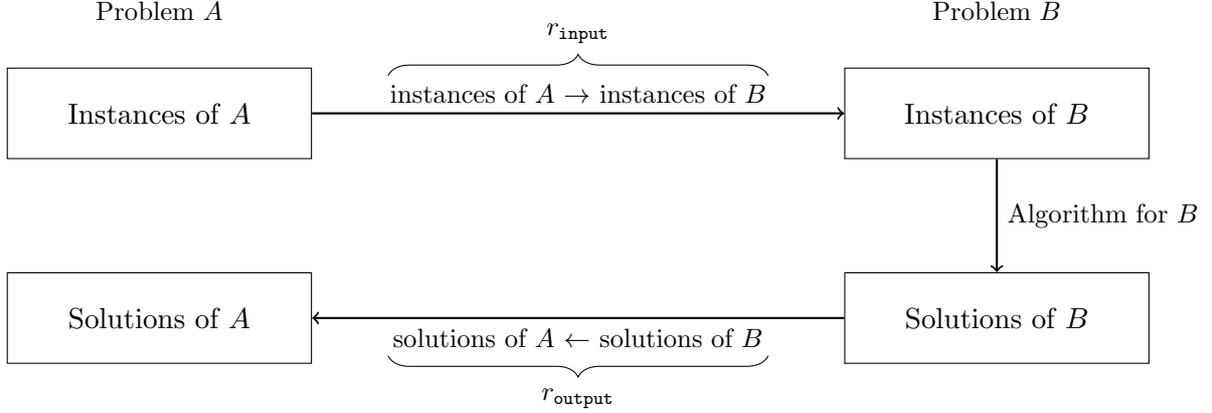

To define reductions, we first introduce a simple abstraction of learning problems that captures both binary classification and stochastic convex optimization, as well as other scenarios. We start by establishing the appropriate terminology, and after that, we will illustrate it in examples.

\begin{framed}
\begin{definition}[Learning task]\label{def:task}
A learning task $\mathcal{T}$ is a tuple $\mathcal{T}=(H, C, Z, P)$, where $H$ is a hypothesis class, $C \subseteq H$ is a concept class, $Z$ is the space of labeled examples, and $P$ is a family of distributions over $Z$. Additionally, each example $z \in Z$ has an associated loss function $\ell_z: H \to \mathbb{R}_{\geq 0}$. For a distribution $D \in P$, the loss of an hypothesis $h \in H$ is given by~$L_D(h) = \Ex_{z \sim D}[\ell_z(h)]$. 
\end{definition}
\end{framed}
For a learning task $\T$, our typical goal is to design a \emph{learning rule} $A$ that maps a sequence of examples $S\in Z^n$ to a hypothesis $A(S) \in H$. The goal is to design a learning rule that competes with the best concept in $C$, such that for every distribution~$D\in P$:
\[\Ex_{S \sim D^n}[L_D(A(S))] \leq \mathtt{OPT}_{C}(D) + o(1),\]
where $\mathtt{OPT}_{C}(D)= \inf_{c \in C} L_D(c)$, and the $o(1)$ term converges to 0 as $n \to \infty$.

A distribution $D$ over $Z$ is called \emph{realizable} if there exists $c \in C$ such that $L_D(c) = 0$; in particular, $\mathtt{OPT}_{C}(D) = 0$. We also say that $D$ is \emph{$\alpha$-realizable}, for $\alpha\geq 0$, if $\mathtt{OPT}_{C}(D) \leq \alpha$.
We say that the task~$\mathcal{T}$ is \emph{realizable} if $P$ is the family of all realizable distributions, and is \emph{agnostic} if it is the family of \emph{all} distributions over $Z$. Unless specified otherwise, we assume the agnostic setting.
As a remark, let us note that the complexity of designing the learning algorithm $A$ is decreasing in $H$ (more options for an output), increasing in $C$ (competing against a bigger class), and increasing in~$P$ (more potential inputs). 

\begin{example}[PAC-learning]\label{ex:pac} In PAC learning model, $C \subseteq \{\pm 1\}^X$ is a class of binary classifiers over a domain $X$ that we want to learn. The labeled examples are $Z = X \times \{\pm 1\}$, and $\ell_{(x,y)}(h) = 1[h(x) \neq y]$ is the 0/1 classification loss. $P$ is the set of realizable distributions in the \emph{realizable} case and the set of all distributions in the \emph{agnostic} case. Finally, $H$ is usually taken to be all functions $H=\{\pm 1\}^X$ restricting it to $H = C$ corresponds to \emph{proper} PAC learning. 
\end{example}

\begin{example}[PAC-learning for partial concept classes]\label{ex:pac-partial} This model generalizes PAC-learning by letting~$C$ be a class of partial binary functions, that is, $C \subseteq \{-1, +1, *\}^X$, where $*$ is treated as ``undefined''. For $c\in C$, the support $\sup(c)$ is the subset of $X$ on which $c$ is defined. Note, however, that the labeled examples are still $Z=X\times \{\pm1\}$, that is, we do not allow examples to be undefined. The loss function is defined similarly, that is, for a partial function $h$ on $X$, $\ell_{(x,y)}(h) = 1[h(x) \neq y]$. In particular, for an example $(x,y)$, the case $h(x) = *$ is treated as a mistake.

The main difference with the classical PAC learning comes from the class of distributions~$P$. Note that now a realizable distribution $D$ should be supported by some partial concept $c\in C$ in the sense that $\Pr_{x,y\sim D}[x\in \sup(c)] = 1$ (assuming $X$ is finite; in the infinite case, there might be additional subtleties). For example, if every concept in $C$ is supported on only half of the domain, then the same is true for any realizable distribution. This contrasts with canonical PAC, where ``typical'' distributions are supported on the whole $X$.
Just as before, $H$ is typically the set of all functions on~$X$, and the agnostic case is defined similarly. 

One particularly useful (and classical) example of PAC learning with partial concept classes is the class of linear classifiers with margins, which we consider in Theorems~\ref{t-prob-sr-margin} and~\ref{t-det-sr-margin}.
\end{example}

\begin{example}[Stochastic convex optimization]\label{ex:sco}
In this setup, $H = C$ is a convex set, which we will call~$W$; $Z$ is an abstract set such that every $z\in Z$ is equipped with a convex function $\ell_z: W \to \mathbb{R}_{\geq 0}$, and $P$ is the set of all distributions over $Z$ (i.e., agnostic setting). All in all, $Z$ might be viewed simply as an index set of a set of convex loss functions, and the goal of the learning algorithm is, having a sample of such functions, to find a point in $W$ that would minimize the expected loss of the given distribution over them.

We will be interested in a \emph{dimension} of an SCO task. For this, 
we will assume that $W\subseteq \R^d$ and refer to the learning task as to \emph{stochastic convex optimization in $\R^d$}.
\end{example}

\begin{example}[General setting of learning]
The general setting of learning, introduced by~\cite*{Vapnik1998, Vapnik1999overview}, can be modeled as a learning task. In this setup, $H = C$ and $Z$ are arbitrary sets, with $P$ being the set of all distributions over $Z$. In general, there are no assumptions on the loss functions $\ell_z: H \to \mathbb{R}_{\geq 0}$.

In many cases within this framework, we assume the existence of a set of examples \( X \) and a set of labels \( Y \), such that \( Z = X \times Y \), and the hypothesis space \( H \) to be a subset of functions from \( X \) to \( Y \) (i.e., \( H \subseteq Y^X \)). In this case, the loss functions take the form \( \ell_{(x, y)}(h) = L\big(h(x), y\big) \), where \( L: Y^2 \to \mathbb{R}_{\geq 0} \) is a fixed loss function that compares the predicted label \( h(x) \) with the true label \( y \). 
\end{example}

For a learning task $\mathcal{T}=(H,C,Z,P)$ and $\alpha \geq 0$, we say that $h\in H$ is \emph{$\alpha$-optimal} for $D\in P$ if $L_D(h) \leq \OPT_C(D) + \alpha$. Note that for a PAC-learning task, $0$-optimality that is witnessed by $c\in C$ is equivalent to realizability.

\begin{framed}
\begin{definition}[Reductions]\label{def:red}
Let $\mathcal{T}_1=(H_1,C_1,Z_1,P_1)$ and $\mathcal{T}_2=(H_2,C_2,Z_2,P_2)$ be two learning tasks, and let $\alpha > 0$ and $\beta\geq 0$. An $(\alpha,\beta)$-reduction $r$ from $\mathcal{T}_1$ to $\mathcal{T}_2$
consists of two maps $r_\inp:Z_1\to Z_2$ and
$r_\outp:H_2\to H_1$ such that the following holds.
\begin{enumerate}
    \item For every distribution $D_1$ in $P_1$, the distribution $r_\inp(D_1)$ is in $P_2$.
    Here $r_\inp(D_1)$ is the push-forward measure induced by sampling $z\sim D_1$ and mapping it to $r_\inp(z)$.
    \item For every $D_1\in P_1$ and $h_2\in H_2$, if $h_2$ is $\alpha$-optimal for $r_\inp(D_1)$ then $r_\outp(h_2)$ is $\beta$-optimal for~$D_1$.
\end{enumerate}

A reduction $r$ is called \underline{exact} if for every distribution $D_1$, realizable by $\mathcal{T}_1$, the push-forward distribution $r_\inp(D_1)$ is realizable by $\mathcal{T}_2$.
\end{definition}
\end{framed}
In the above definition, we do not allow $\alpha=0$ as, potentially, because $\OPT_C(\cdot)$ is defined using infimum, a learning task might not contain $0$-optimal solutions, in which case a notion of $(0, \beta)$-reduction would trivialize; however, allowing $\beta=0$ does not lead to such situations. Note that an $(\alpha,\beta)$-reduction is an $(\alpha',\beta')$ reduction for any $\alpha'\leq \alpha$ and $\beta'\geq \beta$. \Cref{fig:reduction2} below illustrates how $(\alpha, \beta)$-reduction aligns to the approach outlined in \Cref{fig:reduction}.

\begin{figure}[hbt]
    \centering
	    \begin{tikzpicture}[
        box/.style={draw, rectangle, minimum width=4cm, minimum height=1.2cm, align=center}, 
        arrow/.style={->, thick},
        decorate,
        decoration={brace, amplitude=8pt} 
    ]


    \node (P1) at (0,0)  {$P_1$};
    \node (H1) at (0, -2) {$H_1$};

    \node (P2) at (4, 0) {$P_2$};
    \node (H2) at (4, -2) {$H_2$};

    \draw[arrow] (P1.east) -- (P2.west) node[midway, above] {$r_\inp$};
    \draw[arrow] (H2.west) -- (H1.east) node[midway, below] {$r_\outp$};
    
    \draw[dashed] (P2)--(H2) node[midway, right] {\small $\alpha$-optimal};
    \draw[dashed] (P1)--(H1) node[midway, left] {\small $\beta$-optimal};
    
    \draw[double, dashed, ->] (3.5, -1)--(0.5, -1);

%
    \end{tikzpicture}
    \caption{An $(\alpha,\beta)$-reduction.}
    \label{fig:reduction2}
\end{figure}
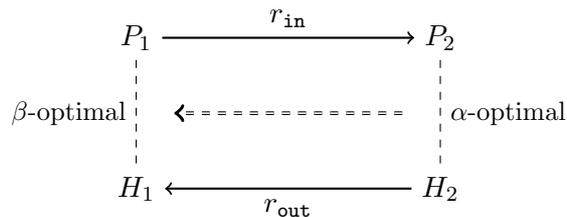

\subsection{Reductions to stochastic convex optimization}\label{sec:sco}

Reductions from classification problems to stochastic convex optimization (SCO) are a common algorithmic tool, exemplified by the use of surrogate losses (e.g., hinge loss), regularization techniques, kernel methods, and other methods that rely on representing data in Euclidean spaces. In this section, we study the minimum dimension in which a concept class $H$ can be reduced to an SCO task. Our main theorem provides a lower bound on the minimum Euclidean dimension required for such reductions, expressed in terms of the VC and dual VC dimension of the problem.

\smallskip\noindent
\begin{minipage}{\textwidth}
\begin{framed}
\begin{theorem}[Binary classification vs. stochastic convex optimization]\label{t:sco}
Let $C\subseteq\{\pm 1\}^X$ be a binary concept class.
If for some $\beta<1/2$ and $\alpha > 0$ there exists an $(\alpha,\beta)$-reduction from the task of learning $C$ in the realizable case to a stochastic convex optimization task in $\mathbb{R}^d$, with loss functions $\{\ell_z\}_{z\in Z}$ satisfying $\ell_z(w)<\infty$ for all $z\in Z$, $w\in W$.	
Then
\[d\geq \max\{\vc(C), \vc^\star(C)-1\}.\]  
\end{theorem}
\vspace{-4mm}
\end{framed}
\end{minipage}\medskip

Perhaps surprisingly, Theorem~\ref{t:sco} implies that certain VC classes require an exponential increase in dimension for learning by reduction to SCO. See Example~\ref{ex:proj} for a simple class that exhibits this property. We discuss more classes that witness boundary cases of \Cref{t:sco} in \Cref{sec:t1-tightness}. 

It is also worth noting that the conclusion in Theorem~\ref{t:sco} extends to certain cases where the loss functions are $\infty$-valued. Below, we provide a natural example of such an $\infty$-valued SCO task (\Cref{ex:LPint}) and present an adaptation of the theorem for these reductions (\Cref{t:sco2}).


From a technical perspective, relating the VC dimension to the Euclidean dimension in SCO is subtle, as they capture fundamentally different aspects of learning complexity: in PAC learning, the VC dimension directly governs sample complexity, whereas in SCO, sample complexity is decoupled from the Euclidean dimension and instead depends on factors such as the Lipschitz constant and the diameter of the parameter space.\footnote{For example, there are one-dimensional SCO problems with an unbounded diameter or Lipschitz constant that are not learnable and infinite-dimensional SCO problems with a bounded diameter and Lipschitz constant that are learnable. (See, e.g.~\cite*{cutkosky2024optml}; Theorems 3.2 and 16.7)} Instead, the Euclidean dimension in SCO is closely related to other resources, such as space complexity, the complexity of arithmetic operations, and test-time complexity. 



\paragraph{VC dimension vs.\ number of parameters.}
There is a natural, intuitive interpretation of the VC dimension as the number of parameters needed to encode concepts in a class, and this intuition is supported by natural classes such as half-spaces and axis-aligned boxes. However, it is known that if one allows general parametrization schemes, this interpretation does not hold. A standard example illustrating this is the class $\{x \mapsto \sign(\sin(tx)): t \in \mathbb{R}\}$, which has an unbounded VC dimension despite being parameterized by a single parameter $t$. (See also \Cref{ex:non-convex-reduction}.) In contrast, Theorem~\ref{t:sco} shows that when the parametrization is restricted to be convex, the number of parameters required must be at least as large as the VC dimension, and in some cases, it must be exponentially larger, see \Cref{ex:proj} in \Cref{sec:t1-tightness} for details. This demonstrates that the VC dimension imposes a meaningful lower bound when parametrizations are constrained to a convex setting. 

Moreover, the open question posed in the previous paragraph—whether the minimal dimension $d$ for which an SCO-reduction exists can be \emph{upper bounded} in terms of the VC dimension—is also of interest here. If such an upper bound exists, it would provide a formal manifestation of the intuition that the VC dimension corresponds to the number of parameters needed to encode and learn the class.

\paragraph{Half-spaces.}
The following examples (Examples~\ref{ex:hingeint} and \ref{ex:LPint}) concern classical reductions from $d$-dimensional linear classification (i.e.\ learning half-spaces) to stochastic convex optimization in $d$-dimensions. 

A half-space is a concept $c_{w, b}: \mathbb{R}^d \to \{\pm1\}$, parametrized by a pair $w \in \mathbb{R}^d$ and $b\in \R$ and defined as $c_{w, b}(x) = \sign(\langle w, x \rangle + b)$. A half-space is called \emph{homogenous} if $b=0$ and \emph{affine} in the general case; if not specified otherwise, we consider half-spaces to be homogenous. Learning half-spaces is a fundamental task in learning theory, and Support Vector Machines (SVMs) are a classical family of algorithms for this problem. SVMs aim to find a consistent half-space whose supporting hyperplane maximizes the margin from the data points. Formally, for each $x \in \mathbb{R}^d$ and $y \in \{\pm 1\}$, the \emph{regularized hinge loss} is defined by
\[
\ell_{x,y}(w) = \lambda\cdot ||w||^2 + \max(0, 1 - y\langle w, x \rangle),
\]
where $\lambda \geq 0$ is a regularization parameter. 
In some cases, the regularized hinge loss is alternatively defined as
\[
\ell_{x,y}(w) = ||w||^2 + c\cdot \max(0, 1 - y\langle w, x \rangle),
\]
where $c \geq 0$ is a regularization parameter. 
Both definitions are equivalent up to a multiplicative constant, by setting $\lambda = 1/c$. Thus, SVMs can be viewed as a reduction from learning half-spaces to stochastic convex optimization, where the loss function is the regularized hinge loss.

In practice, SVMs are used with regularization parameters in $(0, \infty)$. A key advantage that makes SVMs practical is that the regularized hinge loss can be efficiently minimized even on non-separable data (i.e., data from non-realizable distributions). 
Theoretically, SVMs provably learn separable distributions with a margin between positive and negative regions. The next examples show that, in the limit as $\lambda \to \infty$ and $c \to 0$, SVMs reduce $d$-dimensional half-spaces to SCO, even without margin assumptions. These limiting cases are often called \emph{hard SVM}. Interestingly, taking $\lambda$ to infinity and $c$ to zero results in different reductions.


\begin{example}[SVM with unregularized hinge loss]\label{ex:hingeint}
    Let $W=\mathbb{R}^{d+1}$, $Z=\mathbb{R}^d\times\{\pm 1\}$. Define the unregularized hinge loss by
    \[
    \ell_{x,y}(w,a)=\max\big(0,1-y(\langle w,x\rangle+a)\big).
    \]
    This defines the unregularized hinge loss SCO task, which is an SCO problem with continuous loss functions but over a non-compact domain $W$.
    It corresponds to taking $\lambda\to 0$ in the first formula for the regularized hinge loss.
    Define the reduction  $r$ from the task of learning homogenous halfspaces in~$\R^d$ to the above SCO by
    \begin{align*}
    &r_\inp (x,y)=(x,y),
    \\&r_\outp (w, a)=c_{w, a},
\end{align*}
where $c_{w, a}(x)=  \sign(\langle w, x \rangle + a)$.
In Section~\ref{sec:reductions-to-SVM} we show that for all $\alpha>0$, $r$ is an $(\alpha,\alpha)$-reduction.
\end{example}

\begin{example}[Hard SVM]\label{ex:LPint}
Let $W=\mathbb{R}^{d}$, $Z=\mathbb{R}^d\times\{\pm 1\}$ and define the linear programming loss function by
    \[
\ell_{x,y}(w)=\case{||w||^2}{\sign\big(\langle x,w\rangle\big)=y,}{\infty}{\text{otherwise.}}
\]
It corresponds to taking $c\to\infty$ in the second formula for the regularized hinge loss.
Note that while in general $\ell_{x,y}$  is neither continuous nor finite, it is convex. In fact, it can be considered a limit case of the regularized hinge loss function
\[
\ell_{x,y}(w)=||w||^2+c\cdot\max(0,1-y\langle w,x\rangle)
\]
where $c=\infty$.  
This example illustrates how $\infty$ values are useful when one wishes to express hard constraints on the parameter space within the objective function.
This allows to express linear programming as an SCO task. The reduction  $r$ from the learning homogenous halfspaces to this SCO is then defined in the same way as in \Cref{ex:hingeint}.
In Section~\ref{sec:reductions-to-SVM} we show that, for all $\alpha > 0$, this is an $(\alpha,0)$-reduction.
\end{example}

Theorem~\ref{t:sco} does not apply to the last example because the loss functions in the SCO problem are $\infty$-valued. Since, as illustrated by the example above, $\infty$-valued SCO tasks are natural, we prove the next theorem that applies to \emph{exact} reductions to such tasks.

\begin{framed}
\begin{theorem}[A variant of Theorem~\ref{t:sco} for $\infty$-valued SCO]\label{t:sco2}
Let $C\subseteq\{\pm 1\}^X$ be a binary concept class.
Assume that there exists an \underline{exact} reduction from the task of learning $C$ in the realizable case with randomized hypotheses to a stochastic convex optimization task in $\mathbb{R}^d$ with $\infty$-valued loss functions $\ell_z:W\to \mathbb{R}_{\geq0}\cup\{\infty\}$.
Then, 
\[d\geq \max\{\vc(C)-1, \vc^*(C)-1\}.\]  
\end{theorem}
\vspace{-2mm}
\end{framed}

\subsection{Geometric representations}\label{sec:sign-rank}

In this section, we focus on a special type of reductions called \emph{representations}. This notion is quite natural and, within the learning theory, variants of this definition were studied, for example, in \citet*{PittW90,KamathMS20,AliakbarpourB0S24}.

\begin{definition}[Representations]\label{def-repr}
Let $C_1\subseteq \{\pm 1\}^{X_1}$ and $C_2\subseteq \{\pm 1\}^{X_2}$ be two concept classes and let $\alpha\geq 0$. For $\alpha\geq 0$, an \emph{$\alpha$-representation} of $C_1$ by $C_2$ is a map $r:X_1\to X_2$ such that for every distribution $D_1$, realizable by $C_1$, $r(D_1)$ is $\alpha$-realizable by $C_2$.\footnote{Here, with an abuse of notation, we extend $r$ from $X_1$ to $X_1 \times \{\pm1\}$ by letting $r(x, y) = (r(x), y)$, and further extend it to distributions $D_1$ over $X_1 \times \{\pm1\}$ by pushing-forward: sampling from $r(D_1)$ amounts to sampling $(x,y)\sim D_1$ and outputting $r(x,y)$.} The representation $r$ is called \emph{exact} if it is a $0$-representation, that is, it maps realizable distributions into realizable.
\end{definition}

In particular, representations are indeed reductions.
\begin{prop}[Representations are reductions]\label{prop-repr-is-red}
Let $r\colon X_1 \rightarrow X_2$ be an $\alpha$-representation of $C_1$ by $C_2$. Then for $r_\inp(x,y) = (r(x), y)$ and $r_\outp(h)(x) = h(r(x))$, the pair $(r_\inp, r_\outp)$ is a $(\gamma, \gamma + \alpha)$-reduction of the task of PAC learning $C_1$ in a realizable case to PAC-learning $C_2$,
for any $\gamma\geq 0$.
\end{prop}

The representations in the above sense were defined in \cite{PittW90}. One of the most popular reductions, especially in the context of studying the expressivity of kernel methods, is the reduction to the class of half-spaces. The dimensionality of this reduction is also known as a \emph{sign-rank}. We follow \cite{KamathMS20} in its formal definition below (see their Definition 1). Note that there the authors consider this definition with different loss functions and use terms \emph{dimension complexity} for the generalized situation and \emph{sign-rank} specifically for the case of $0/1$ loss. 

\begin{definition}[Sign-rank]
For a class $C\subseteq \{\pm1\}^X$, the \emph{sign-rank} of $C$ is the smallest $d$ for which there exists an exact representation of $C$ by homogenous half-spaces in $\R^d$. That is, such that there are maps $\varphi\colon X\rightarrow \R^d$ and $w\colon C\rightarrow \R^d$ such that for all $c\in C$ and $x\in X$, it holds $\sign \langle w(c), \varphi(x) \rangle = c(x)$.
\end{definition}

Sign-rank and the reductions to SCO are closely related. Examples~\ref{ex:hingeint} and~\ref{ex:LPint} above show that a class of sign-rank at most $d$ can be reduced to SCO in dimension $d$. 
The following result shows that, for finite classes and exact reductions, this almost (with a factor of $+1$) goes in the other direction. We note that $+1$ can be removed if we allow representations by affine halfspaces, rather than with only homogenous, and so in the affine case this connection becomes tight.

\begin{framed}
\begin{theorem}[Half-spaces are complete for exact reductions to SCO]\label{t:halfspaces-complete}
Let $C\subseteq\{\pm 1\}^X$ be a finite concept class.
If for some $\beta<1/2$ and $\alpha$ there exists an exact $(\alpha,\beta)$-reduction from the task of learning $C$ in the realizable case to a stochastic convex optimization task in $\mathbb{R}^d$, then $C$ has an exact representation by homogenous half-spaces in $\mathbb{R}^{d+1}$. In other words, the sign-rank of~$C$ is at most~$d+1$.
\end{theorem}
\vspace{-2mm}
\end{framed}

Note that our setup naturally enables us to consider approximate reductions and representations, where we understand ``approximate'' as ``not exact''. Note that for representations it simply means letting $\alpha > 0$, while in the context of reductions, it means something else, which does not boil down to the parameters $\alpha$ and $\beta$. Formally:
\begin{definition}[Approximate sign-rank and SCO dimension]
For a class $C$ and $\alpha < 1/2$ we define an \emph{$\alpha$-sign-rank} of $C$ as the smallest $d$ for which $C$ has an $\alpha$-representation by homogenous half-spaces in $\R^d$. Similarly, we define an \emph{SCO-dimension} of $C$ as the smallest $d$ for which there is an $(\alpha, \beta)$-reduction, for some $\beta < 1/2$, from the realizable learning of $C$ to an SCO task in~$\R^d$.
\end{definition}

While \Cref{t:halfspaces-complete} says that, in the exact case, the SCO dimension and the sign-rank are essentially equivalent, we know much less about the approximate case. With respect to (approximate) SCO versus the (exact) sign-rank, from the above, we know that 
\begin{align*}
\max\{\mathtt{VC},\mathtt{VC}^\star-1\}\leq \text{SCO} \leq \text{sign-rank}.
\end{align*} 
It is also known that there are classes with constant VC dimension and unbounded sign-rank~\citep{Ben-DavidES02,Alon17}. However, we do not know if SCO can be substantially smaller than the sign rank, and whether it can be upper bounded by a function of the VC dimension or lower bounded by a function of the sign-rank.

At the same time, the following theorem establishes a separation between the exact and approximate sign-ranks.
\begin{framed}
\begin{theorem}[Exact vs approximate sign-rank]\label{t:approxVSexact}
For every integer $d\geq 0$, there exists a finite concept class $C_d$ whose $(1/3)$-sign-rank is at most $d$, and the sign-rank is at least~$d^{\Omega(\log d)}$.
\end{theorem}
\vspace{-2mm}
\end{framed}
Note that the class in \Cref{t:approxVSexact} is rather simple: its domain is $\R^d$ and the functions are majority votes of the signs of three homogenous halfspaces in $\R^d$.

A natural way to generalize representations is to allow them to be probabilistic. 
\begin{definition}[Randomized representations]\label{def-repr-rand}
A \emph{$\delta$-confident randomized $\alpha$-representation} of $C_1$ by~$C_2$, or $(\delta, \alpha)$-representation for short, is a distribution $R$ over maps $r:X_1\to X_2$ such that for every distribution $D_1$, realizable by $C_1$, the distribution $r(D_1)$ is $\alpha$-realizable by $C_2$ with probability at least $1-\delta$ over $r\sim R$.
\end{definition}

As before, the most canonical representation is the one by half-spaces in $\R^d$, which gives rise to the definition of a randomized sign-rank$^*$

\begin{definition}[Probabilistic sign-rank$^*$]\label{def-sr-prob-star}
For a class $C$, $\delta>0$, and $\alpha < 1/2$, we define a \emph{probabilistic $\delta$-confident $\alpha$-sign-rank$^*$} of $C$, or $(\delta, \alpha)$-sign-rank$^*$ for short, as the smallest $d$ for which $C$ has a $(\delta,\alpha)$-representation by homogenous half-spaces in $\R^d$. 
\end{definition}

Similar notions of \emph{probabilistic dimension complexity} and \emph{probabilistic sign-rank} were studied in the same paper~\citep{KamathMS20}, see their Definitions~2 and~23. The latter is also a well-known term in communication complexity, see~\citep*{Alman17}. However, due to an additional swerve space from randomization, neither of them is identical to ours; in particular, we write sign-rank$^*$ to separate our notion from the established term in an adjacent area. 

Theorems~\ref{t-prob-sr-margin} and~\ref{t-det-sr-margin} below establish a separation between the probabilistic and the deterministic sign-ranks for a partial class of linear classifiers over $\S^n$ with margin. Note that the bounds on the probabilistic sign-rank in \Cref{t-prob-sr-margin} are applicable (and stated) not only for our sign-rank$^*$, but also for the two of its abovementioned relatives, see the discussion and the definitions in \Cref{sec-prob-sign-ranks}.
\begin{framed}
\begin{theorem}[Probabilistic sign-rank of halfspaces with margin]\label{t-prob-sr-margin}
Let $C_{n}$ be the partial class of linear classifiers with constant margin $\gamma=1/3$ on the $n$-dimensional sphere $\mathbb{S}^n$. Then for $\alpha \in (0, 1/2)$ and $\delta\in (0,1)$, the $(\delta, \alpha)$-sign-rank$^*$ $d$ of $C_{n}$ is at most
    \[d= O\Bigl(\log\frac{1}{\alpha\delta}\Bigr).\]
The randomized $(\delta, \alpha)$-representation witnessing it is linear, that is, the respective distribution is over linear maps $\S^n\rightarrow \R^d$, where we treat $\S^n$ as a unit sphere in~$\R^{n+1}$.

Moreover, probabilistic $\delta$-dimension complexity and probabilistic $\delta$-sign-rank are at most $O(\log(1/\delta))$.
\end{theorem}
\vspace{-3mm}
\end{framed}

\smallskip\noindent
\begin{minipage}{\textwidth}
\begin{framed}
\begin{theorem}[Deterministic sign-rank of halfspaces with margin]\label{t-det-sr-margin}
Let $C_{n}$ be the partial class of linear classifiers with constant margin $\gamma=1/3$ on the $n$-dimensional sphere $\mathbb{S}^n$. Then for $\alpha \in (0, 1/2)$ the (deterministic) $\alpha$-sign-rank $d$ of $C_n$ is at least
    \[d \geq \min\Bigl\{\frac{1-\alpha}{\alpha},n+1\Bigr\}.\]
Moreover, if the respective $\alpha$-representation is continuous, then $d \geq n+1$.
\end{theorem}
\end{framed}
\end{minipage}\medskip

In~\cite{KamathMS20}, the authors ask whether there is an ``infinite'' separation between the probabilistic dimension complexity and the sign-rank. Modulo the fact that our class $C_n$ is partial, the above separation gives a positive answer to this question. Indeed, by \Cref{t-prob-sr-margin}, the $\delta$-dimension complexity of the family of classes $C_n$ above is uniformly bounded, independently of $n$. At the same time, by \Cref{t-det-sr-margin}, the (exact deterministic) sign-rank of $C_n$ is unbounded (although the case of exact sign-rank is also known from Theorem~1.5 in~\cite{HatamiHM22}). Moreover, let
$C=\bigcup_{n=1}^\infty C_n$; here we abuse the fact that our classes are partial and that $C$ is defined on a disjoint union of domains of the classes~$C_n$. This way, a formal union of $(\delta, \alpha)$-representations of $C_n$'s gives a $(\delta, \alpha)$-representation of~$C$. Hence, $C$ is a partial class of bounded $\delta$-dimension complexity and infinite sign-rank. 

At the same time, many natural questions about the sign-rank and its relaxations can further be asked. We know that, for fixed $\alpha \in (0, 1/2)$ and $\delta\in (0,1)$
\[
(\delta, \alpha)\text{-sign-rank}^* \leq \alpha\text{-sign-rank} \leq \text{sign-rank}.
\]
As argued, for a fixed $\delta$ and $\alpha$, there is an infinite separation between the first and the last one. Moreover, for a fixed $\delta$, there is an exponential in $1/\alpha$ separation between the first and the second. 
It is, however, left open whether any two of the ``adjacent'' ranks above are finite together. In particular, we conjecture that, for a fixed $\alpha$, the constant lower bound in~\Cref{t-det-sr-margin} is too weak, and the $\alpha$-sign-rank of $C_n$ goes to infinity with $n$. This, in particular, would imply an infinite separation between the probabilistic and the approximate sign-ranks.

\subsection{A variant of Borsuk-Ulam for closed convex relations}
Our proof of \Cref{t:sco} uses a version of the Borsuk-Ulam theorem that combines the classical topological approach with convexity considerations.
Recall that the original Borsuk-Ulam theorem states that a sphere $\S^d$ cannot be continuously mapped to $\R^d$ without collapsing a pair of antipodal points.
Borsuk-Ulam theorem is a well-established tool in combinatorics, brought to light by Lov{\'a}sz's proof of Kneser conjecture~\citep{lovasz78}, but which since then developed into a central tool in topological combinatorics, see~\citep*{matouvsek03}. More recently, it has found applications in TCS and learning theory, see~\cite{HatamiHM22}, \cite*{chase2024dual, chase2024local}.
\begin{framed}
\begin{theorem}[Borsuk-Ulam for closed convex relations]\label{t:Borsuk-Ulam}
Let $W$ be a compact convex set in $\mathbb{R}^k$ and let $G$ be a closed set inside $\mathbb{S}^d \times W$, where $\mathbb{S}^d$ is a $d$-dimensional unit sphere. Additionally, suppose that:
\begin{itemize}
    \item For any $x\in \mathbb{S}^d$, the set $G_x = \{w\in W~|~(x,w)\in G\}$ is nonempty and convex;
    \item For any $x\in \mathbb{S}^d$, the sets $G_x$ and $G_{-x}$ are disjoint.
\end{itemize}	
Then $k\geq d+1$.
\end{theorem}    
\vspace{-3mm}
\end{framed}
The original Borsuk-Ulam theorem can be seen as a special case of \Cref{t:Borsuk-Ulam} by setting $G = {(x, f(x)) : x \in \mathbb{S}^d}$ as the graph of $f$ and $W$ as the convex hull of the image of $f$. 

While formally for the proof of \Cref{t:sco} we use not \Cref{t:Borsuk-Ulam} itself, but rather its close relative \Cref{t:BU-open}, the latter is technical and its statement is less self-sufficient.

\section{Proof overview}\label{sec:proofoverview}
Most of the proofs use, in one or another way, the standard toolset of the learning theory, such as the VC dimension, the sign-rank, loss estimates, etc., see \Cref{sec:prelim} below for a brief overview. In particular, the central definitions that we introduce and explore, namely, learning task (\Cref{def:task}), reduction (\Cref{def:red}), and representation (\Cref{def-repr}), are quite typical, although to our knowledge, our formalization of them is new. Below, we outline the use of techniques that are either uncommon for this area, or that are in some sense specialized.

The main technical ingredient in the proof of reduction from classification to SCO (\Cref{t:sco}) is a version of Borsuk-Ulam for closed convex relations (\Cref{t:Borsuk-Ulam}). The proof of the latter relies on the classical Borsuk-Ulam together with some tools from convex geometry (Carath\'eodory's theorem) and real analysis (partitions of unity), although the proof can be carried out within a more topological toolset by using dense triangulations instead of the last two. The proof of \Cref{t:sco} itself then uses the topological approach to PAC-learning problems similar to the one in~\cite{chase2024local}, see the proof of Theorem~D there. Namely, the set of realizable distributions of a class is topologized by equipping it with a total variation metric and the antipodality that comes from flipping the signs of the labels. The application of a Borsuk-Ulam-like theorem then hinges on two crucial observations: i) that no hypothesis can simultaneously achieve loss $<1/2$ on two antipodal distributions, and
ii) that both VC and dual VC dimensions enforce the existence of spheres of comparable dimension in the space of realizable distributions.

We would like to remark that deriving \Cref{t:sco} does not seem to follow from standard sample-complexity-based considerations. While this result lower-bounds the dimension of the SCO task in terms of the VC dimension of the class, reduced to it, it would be interesting to determine whether, conversely, the minimal dimension $d$ for which a VC class $H$ is reducible to SCO in $\mathbb{R}^d$ can be \emph{upper bounded} in terms of the VC dimension. Additionally, it would be interesting to explore whether a more elementary proof of \Cref{t:sco} can be found.

The completeness of halfspaces for exact reductions to SCO (\Cref{t:halfspaces-complete}) is proven using LP duality for a certain game played on the space $W$, associated with the SCO task. The result about exact reductions to $\infty$-valued SCO (\Cref{t:sco2}) is an easy consequence of it and of the fact that both VC and dual VC lower-bound the sign-rank~\cite{Alon17}.

The separation between exact and approximate sign-ranks(\Cref{t:approxVSexact}) is witnessed by a specific class, whose approximate sign-rank is easily upper-bounded, and sign-rank was shown to be high by~\citet*{bun21}. For the other separation, between the probabilistic and approximate sign-ranks (Theorems~\ref{t-prob-sr-margin} and~\ref{t-det-sr-margin}), the probabilistic sign rank of the class in question is upper-bounded using a relative of the Johnson-Lindenstrauss theorem, and its approximate sign-rank is lower-bounded by combining uniform convergence (see \Cref{t:UnifromConvergence} below) with the result in~\cite{HatamiHM22}, establishing the sign-rank for this class; interestingly enough, the latter is also proved by a topological argument that uses Borsuk-Ulam.

\section{Examples}\label{sec:examples}
\subsection{Examples demonstrating the tightness of Theorem~\ref{t:sco}}\label{sec:t1-tightness}

The following is an example of a simple class for which the dual VC dimension, and hence the SCO dimension, is exponential in the VC dimension.
\begin{example}[Projection functions] \label{ex:proj}
Let $d$ be the dimension, and consider the domain $X = \{\pm 1\}^d$ and the class 
$U_d = \{h_i : i \leq d\}$
of all projection functions defined by $h_i(x_1, \ldots, x_d) = x_i$ on $(x_1,\ldots,x_d)\in X$.  
This class satisfies: (i) $\vc(U_d)\leq \log d$, (ii) $\vc^\star(U_d) = d$, and (iii) $U_d$ is reducible to stochastic convex optimization in $\mathbb{R}^d$. This is because $U_d$ can be represented by half-spaces in $\mathbb{R}^d$: we embed $X$ into $\mathbb{R}^d$ trivially (since $X \subseteq \mathbb{R}^d$) and represent each concept $h_i$ as the half-space $(x_1, \ldots, x_d) \mapsto \sign(x_i)$. 
\end{example}

From the PAC learning perspective, $\beta$-optimality for $\beta < 1/2$ corresponds to weak learnability. That is, a hypothesis $h$ which is $(<1/2)$-optimal for a distribution $D$ provides some information about~$D$. On the other hand, $1/2$-optimality can always be achieved by a random guess, and hence carries no information about the distribution. In the world of reductions, this is paralleled by the next two examples that show that a learning task can always be reduced to a $1$-dimensional SCO for $\beta=1/2$, as long we allow our hypotheses to be randomized.

\begin{example}[PAC-learning with randomized hypotheses]\label{ex:pac-rand}
In the setup of \Cref{ex:pac}, it is often useful to allow a learner to use randomized hypotheses. We model this by defining $H=[-1,1]^X$ where a randomized hypothesis $h: X\to [-1,1]$ on $x$ outputs $1$ with probability $\frac{1+h(x)}{2}$ and $-1$ with probability $\frac{1-h(x)}{2}$. The loss function is then defined by $\ell_z(h)=\frac{1}{2}|h(x)-y|$, which corresponds to the expected 0/1 loss of $h$.  

Note that $\ell_z(h)$ on a ``deterministic'' hypothesis $h$ with values in $\{\pm1\}$ is just a usual 0/1 loss, and hence this model generalizes the classical PAC setting. While this approach is less used in the PAC-learning paradigm, this generalization is common, for example, in the online learning setting, see Section~1.2.2 in~\cite{MAL-018}. For us, it is useful in the context of reductions from PAC-learning to stochastic convex optimization, see~\Cref{ex-trivial-reduction} below.  
\end{example}

\begin{example}[A trivial reduction]\label{ex-trivial-reduction}
This example demonstrates that any concept class $C \subseteq \{\pm 1\}^X$ can be reduced to a one-dimensional stochastic convex optimization problem with a classification error of $\beta=1/2$. This highlights the tightness of the assumption $\beta<1/2$ in Theorem~\ref{t:sco}.

Consider the task of learning a concept class $C \subseteq \{\pm 1\}^X$ in the realizable setting using a randomized learning rule. So, $H = [-1,1]^X$, where each $h \in H$ represents a random hypothesis that outputs $1$ with probability $\frac{1+h(x)}{2}$ and $-1$ with probability $\frac{1-h(x)}{2}$. The loss is the expected zero-one loss, $\ell_{(x,y)}(h) = \frac{1}{2} \lvert h(x) - y \rvert$.

We reduce to a stochastic convex optimization task with $W = [-1, 1]$, $Z = \{\pm 1\}$, and $\ell_z(w) = \frac{1}{2} \lvert z - w \rvert$. The reduction is defined by $r_\inp(x, y) = y$ and $r_\outp(w) = h_w$, where $h_w(x) = w$ for all $x \in X$. 
We claim that for $\alpha\in[0,1/2]$, the above is an $(\alpha,\frac{1+\alpha}{2})$-reduction. In particular, for $\alpha=0$ it is a $(\alpha=0,\beta=\frac{1}{2})$-reduction.

\begin{proof}Let $D_1$ be a distribution realizable by $C$ and $D_2 = r_{\inp}(D_1)$.
Notice that $L_{D_1}(h_w) = L_{D_2}(w)$ for all $w\in[-1,+1]$.
Let $w^\star\in  W$ be $\alpha$-optimal with respect to $D_2$.
We need to show that $r_\outp(w^\star)=h_{w^\star}$ is $\frac{1+\alpha}{2}$-optimal with respect to $D_1$. 
Define $p_+=\Pr_{(x,y)\sim D_1}[y=+1]$ and $p_- = \Pr_{(x,y)\sim D_1}[y=-1]$. 
For all $w\in W$:
\[
L_{D_2}(w)=\frac{p_+|1-w|+p_-|1+w|}{2}=\frac{1-(p_+-p_-)w}{2}.
\] 
Optimizing for the above yields that minimum of $L_{D_2}$ on $W$ is $\min\{p_+,p_-\}$. 
Hence, since $w^\star$ is $\alpha$ optimal for $D_2$, then $L_{D_1}(h_{w^\star}) = L_{D_2}(w^\star)\leq \alpha+\min\{p_+,p_-\}$. And since the maximum of $L_{D_2}$ is $\max\{p_+,p_-\}$, this bound can be improved to 
\begin{align*}
L_{D_1}(h_{w^\star})&\leq \min\bigl\{\alpha+\min\{p_+,p_-\},\max\{p_+,p_-\}\bigr\}\\
&\leq \frac{\alpha + \min\{p_+,p_-\} + \max\{p_+,p_-\}}{2} 
= \frac{1+\alpha}{2}.
\end{align*}
\end{proof}
\end{example}

The following example demonstrates that if we remove the requirement for convex loss functions, any finite concept class can be reduced to a one-dimensional learning task with continuous loss functions.
\begin{example}[A non-convex reduction to one dimension]\label{ex:non-convex-reduction}
Consider the task of learning a finite concept class $C \subseteq \{\pm 1\}^X$ in the realizable setting using a randomized learning rule. Define the following learning task: Let  $Z=X\times\{\pm 1\}$, $W=[0,1]$, and for each $c\in C$ pick some unique $w_c\in W$. For each example  $(x,y)\in Z$ define the set $V_{x,y}=\{w_c\in W\;:\; c(x)=y\}$. Now define the loss function $\ell_{x,y}$ by 
\[
\ell_{x,y}(w)=\frac{d(w,V_{x,y})}{d(w,V_{x,y})+d(w,V_{x,-y})}.
\]
Note that $\ell_{x,y}$ is continuous (but not convex).
Define the following reduction from $C$ to the above learning task
\begin{align*}
    &r_\inp(x,y)=(x,y),
\\&r_\outp(w)(x)=\sign \big(\ell_{x,-1}(w)-\ell_{x,1}(w) \big).
\end{align*}
We claim that for all 
$\alpha>0$ the above is an exact $(\alpha,2\alpha)$-reduction.
\begin{proof}
The above reduction is exact because $L_{r_\inp(D)}(w_c)=0$ for any distribution $D$ which is realizable by $c\in C$. To show it is $(\alpha,2\alpha)$-reduction it is enough to prove that \[
L^{01}_D\big(r_\outp(w)\big)\leq 2L_{r_\inp(D)}(w).
\]
Note that $\ell_{x,-1}(w)+\ell_{x,1}(w)=1$. Thus, $r_\outp(w)(x) \neq y$ implies $l_{x, y}(w)\geq 1/2$. But then
\begin{align*}
    L^{01}_D&\big(r_\outp(w)\big) =
        \Ex_{(x,y)\sim D} \left[r_\outp(w)(x)\neq y\right] \\
        &\leq \Ex_{(x,y)\sim D} 2l_{x, y}(w)
        = 2 \Ex_{(x,y)\sim r_\inp(D)} l_{x, y}(w) = 2L_{r_\inp(D)}(w),
\end{align*}
as needed.
\end{proof}
\end{example}

\subsection{Reductions to SVM}\label{sec:reductions-to-SVM}
In this subsection, we prove statements about the reductions to SVM from \Cref{ex:hingeint} and \Cref{ex:LPint}. We restate the above examples in an abridged form, to make apparent the claims that we prove.

\begin{customexample}{\ref{ex:hingeint}}[SVM with unregularized hinge loss]
	This SVM is an SCO task with $W=\mathbb{R}^{d+1}$, $Z=\mathbb{R}^d\times\{\pm 1\}$, and, for $(w, a)\in \R^d \times \R = W$,
    \[
    \ell_{x,y}(w,a)=\max\big(0,1-y(\langle w,x\rangle+a)\big).
    \]
    Define the reduction  $r$ from the task of learning affine halfspaces in $\R^d$ to this task by
    \begin{align*}
    &r_\inp (x,y)=(x,y),
    \\&r_\outp (w,a)=c_{w,a},
\end{align*}
where $c_{w,a}(x) =  \sign(\langle w, x \rangle+a)$.
Then for any $\alpha>0$, $r$ is an $(\alpha,\alpha)$-reduction.
\end{customexample}
\begin{proof}
%
Recall that $\ell^{01}_{x,y}(w,a) = [\sign(\langle w,x\rangle+a)\neq y]$. In particular, $\ell^{01}_{x,y}(w,a) = 1$ means that $y(\langle w,x\rangle+a) \geq 0$ and hence $\ell_{x,y}(w,a)\geq 1$. So, for any realizable $D$ and $w, a\in W$, $L^{01}_D(c_{w,a})\leq L_{D}(w,a)$.
Thus, to prove the validity of the reduction, it suffices to show that $\OPT_W(D)=0$ for all $D$ realizable by $d$-dimensional half-spaces. 

Let $D$ be a distribution realizable by $c_{w,a}$ 
and let $\varepsilon>0$, we claim that for $n$ large enough we have
\[
L_D(nw,na +1)<\varepsilon.
\]

Indeed, 
\begin{align*}
L_D&(nw,na +1) = \Ex_{x, y\sim D} \max\bigl(0, 1 - y(\langle nw, x\rangle + na + 1)\bigr) \\
	&\leq \Ex_{\substack{x, y\sim D\\y(\langle nw, x\rangle + na + 1) < 1}} \max\bigl(0, 1 - y(\langle nw, x\rangle + na + 1)\bigr) \\
	&= \Ex_{\substack{x, y\sim D\\ny(\langle w, x\rangle + a)  < 1 - y}} \max\bigl(0, 1 - y(\langle nw, x\rangle + na + 1)\bigr) \\
	&= \biggl[y(\langle w,x\rangle+a) \geq 0 \text{ for } x, y\sim D \biggr]\\
	&= \Ex_{\substack{x, y\sim D;~~y = -1\\0\leq -n(\langle w, x\rangle + a)  < 2}} \max\bigl(0, 1 + 1(\langle nw, x\rangle + na + 1)\bigr) \\
	&= \Ex_{\substack{x, y\sim D;~~y = -1\\-2 < n(\langle w, x\rangle + a)  \leq 0}} \max\bigl(0, 2 + n(\langle w, x\rangle + a)\bigr) 
	\leq \Ex_{\substack{x, y\sim D;~~y = -1\\-2 < n(\langle w, x\rangle + a)  \leq 0}} 2 \\
	&= 2 \Pr_{x, y\sim D} \bigl[ -2/n < \langle w, x\rangle + a \leq 0  \text{ and } y=-1 \leq 0 \bigr]\\
	&\xrightarrow[n\rightarrow \infty]{} 2 \Pr_{x, y\sim D} \bigl[\langle w, x\rangle + a  = 0\text{ and } y=-1 \leq 0 \bigr] = 0.
\end{align*}
\end{proof}

\begin{customexample}{\ref{ex:LPint}}[Hard SVM]
	This SVM is an SCO task with $W=\mathbb{R}^{d}$, $Z=\mathbb{R}^d\times\{\pm 1\}$, and
    \[
	\ell_{x,y}(w)=\case{||w||^2}{\sign\big(\langle x,w\rangle\big)=y,}{\infty}{\text{otherwise.}}
    \]
    Define the reduction  $r$ from the task of learning homogenous halfspaces in $\R^d$ to this task by
    \begin{align*}
    &r_\inp (x,y)=(x,y),
    \\&r_\outp (w)=c_{w},
\end{align*}
where $c_{w}(x) = c_{w, 0}(x) =  \sign(\langle w, x \rangle)$.
Then for any $\alpha>0$, $r$ is an $(\alpha,0)$-reduction.
\end{customexample}
\begin{proof}
    Note that for every distribution $D$ which is realizable by half-spaces and every $w\in W$ we have that $L_{r_\inp(D)}(w)$ is $||w||^2$ if $\Pr_{(x,y)\sim D}[\sign\big(\langle x,w\rangle\big)=y]=1$ and infinite otherwise. 
    Hence, for any finite~$\alpha$, every $\alpha$-optimal solution for $L_D$ achieves a loss of zero on the zero-one loss function $L_D^{01}$. Thus,
    $r$ is indeed an  $(\alpha,0)$-reduction for all $\alpha>0$. And, by rescaling $w$ we get that $\OPT_W(D)=0$ so the reduction is exact.
\end{proof}

\section{Preliminaries}\label{sec:prelim}

We use standard notation from learning theory, see e.g.\ \cite{shays14}. 
 In general, $X$ will denote the domain and $C$ a concept class of functions from $X$ to $\{\pm 1\}$. For a given collection of loss functions $\{\ell_z\}_{z\in Z}$ for each  distribution $D$ over $Z$ and finite sample $S\subseteq Z$ we define the induced loss functions \begin{align*}
     &L_D(w)=\Ex_{z\sim D}\ell_z(w),
     \\& L_S(w)=\frac{1}{|S|}\sum_{z\in Z}\ell_z(w).
 \end{align*}
 Whenever there can be some confusion with other loss functions, we will use $\ell_{x,y}^{01}(c)$ (and similarly $L_D^{01}$) to denote the zero-one loss function which is $0$ if $c(x)=y$ and $1$ otherwise.

\begin{definition}[VC dimension]
We say that a set $\{x_i\}_{i=1}^n$ is shattered  by a concept class $C$ if for any labeled sample $S=\{(x_i,y_i)\}_{i=1}^n$ over this set  there is some $c\in C$ such that $L_S(c)=0$. The VC dimension $\vc(C)$ of $C$ is the largest number $n$ such that $C$ shatters a set of size $n$, or infinity if $C$ shatters sets of arbitrary size.
\end{definition}

\begin{definition}[Dual VC dimension]
For a fixed concept class $C$ over a domain $X$, each $x\in X$ defines a function $h_x :C \to \{\pm1\}$ by $h_x(c)=c(x)$. The class $C^\star=\{h_x \;:\; x\in X\}$ is called a dual class of $C$. The dual VC dimension of $C$ is the VC dimension of $C^\star$, $\vc^\star(C)=\vc(C^\star)$.
\end{definition}

We say that a  concept class $C$ satisfies uniform convergence if there exists a vanishing sequence $\eps_n\xrightarrow{n\to\infty} 0$ such that for all distributions $D$ over $X\times\{\pm1\}$,
\[\Ex_{S\sim D^n}\Big(\sup_{c\in C}\lvert \mathtt{L}_D(c) - \mathtt{L}_S(c) \rvert\Big)\leq \eps_n.\]
The seminal work by~\cite*{VC71} shows that any concept class $C$ with finite VC dimension satisfies the uniform convergence.  The following asymptotically optimal quantitative bound was achieved in the seminal work by~\cite*{Talagrand94chaining} using a technique called chaining~\citep*{dudley:78}.

\begin{theorem}[Uniform convergence for VC classes]\label{t:UnifromConvergence}
Let $C$ be a concept class with finite VC dimension $\vc(C)=d$. Then for any $n>0$ and a distribution $D$ over $X\times \{\pm1\}$ we have
\[
\Ex_{S\sim D^n}\Big(\sup_{c\in C}|L_D(c)-L_S(c)|\Big)=O\bigl(\sqrt{d/n}\bigr).
\]
    
\end{theorem}

\section{Proofs for reductions to SCO}\label{sec-proofs-sco}
Our aim for this section is to prove \Cref{t:sco} and \Cref{t:halfspaces-complete}. However, a central technical ingredient here is \Cref{t:BU-open} below, which is a close relative of the Borsuk-Ulam for closed convex relations (\Cref{t:Borsuk-Ulam}), but with a more technical statement.
For accessibility, in this and the following sections we restate the theorems that we are going to prove.

\begin{theorem}[Borsuk-Ulam for certain relations]\label{t:BU-open}
Let $W$ be a compact convex set in $\mathbb{R}^k$ and let $G\subseteq \mathbb{S}^d \times W$. Additionally, suppose that:
\begin{itemize}
    \item For any $x\in \mathbb{S}^d$, the set $G_x = \{w\in W~|~(x,w)\in G\}$ is nonempty;
    \item For any $x\in \mathbb{S}^d$, the convex hulls of $G_x$ and $G_{-x}$ are disjoint;
    \item For any $w\in W$, the set $G^w =\{x\in \mathbb{S}^d \;:\; (x,w)\in G\}$ is open.
\end{itemize}	
Then $k\geq d+1$.
\end{theorem}    
\begin{proof}
Note that $x\in G^w$ if and only if $y\in G_x$, and, as $G_x$ is nonempty for all $x$, the family $\{G^w\}_{w\in W}$ is an open cover of $\mathbb{S}^d$. By compactness, there is a finite $T\subset W$ such that  $\{G^t\}_{t\in T}$ is a finite open cover of $\mathbb{S}^d$. Let $\{\rho_t\}_{t\in T}$ be a partition of unity subordinate to this cover, that is, a family of continuous functions parameterized by $t\in T$ such that each $\rho_t$ is $0$ outside of $G^t$ and for all $x\in \mathbb{S}^d$, 
$\sum_{t\in T}\rho_t(x)=1$. Partitions of unity are a well-known tool in real analysis, see Theorem 2.13 in~\cite{rudin87}. We can also explicitly define
\[
\rho_t(x)=\frac{d(x,\mathbb{S^d}\setminus G^t)}{\sum_{t'\in T}d(x,\mathbb{S^d}\setminus G^{t'})}.
\]
Define $\Phi : \mathbb{S}^d\to W$ by \[
\Phi(x)=\sum_{t\in T}\rho_t(x)t.
\]
Note that $\rho_t(x)>0$ implies that $t\in G_x$, so $\Phi(x)$ is a convex combination of elements in $G_x$. Then, as $\Phi(x)$ and $\Phi(-x)$ are in the convex hulls of $G_{x}$ and $G_{-x}$ respectively and the latter are disjoint, $\Phi(x)\neq \Phi(-x)$. By the Borsuk-Ulam theorem, this implies $k\geq d+1$, as needed.
\end{proof}

\begin{lem}\label{l:Inf-Lin-Contin}
    Let $I\subset \mathbb{R}^d$ be an nonempty collection of pointwise positive vectors in $\mathbb{R}^d$, that is, such that $\alpha_i >0$ for all $\alpha\in I$ and $i=1, \dots, d$. Let $X=\{x\in [0,1]^d\;:\; \sum_{i=1}^d x_i=1\}$ and define $F:X\to \mathbb{R}$ by
    \[
    F(x)=\inf_{\alpha\in I}\inProd{x}{\alpha}.
    \]
    Then $F$ is continuous.
\end{lem}
\begin{proof}
    For $x\in X$, let us take an arbitrary $\varepsilon>0$ and let $\delta>0$ to be specified later. Define $U_x$ to be an open neighborhood of $x$ containing the points $y$ such that $|x_i-y_i|< \delta x_i$ whenever $x_i\neq 0$ and for $y_i<\delta$ whenever $x_i = 0$. Let $y\in U_x$, we will show that for $\delta$ small enough $|F(x)-F(y)|<2\varepsilon$. Fix some 
    $\alpha\in I$ such that 
    \[
   \inProd{x}{\alpha}< F(x)+\varepsilon.
    \]
    First, we bound $F(y)-F(x)$. Note that  
    \[
    \inProd{y}{\alpha}-\inProd{x}{\alpha}\leq \sum_{i=1}^d \lvert x_i-y_i \rvert \alpha_i\leq \sum_{x_i\neq 0} \delta x_i\alpha_i+\sum_{x_i=0} \delta \alpha_i \leq \delta\big(F(x)+\varepsilon+||\alpha||_1\big).
    \]
    So, 
    \[
    F(y)\leq \inProd{y}{\alpha}\leq \inProd{x}{\alpha}+\delta\big(F(x)+\varepsilon+||\alpha||_1\big)\leq F(x)+\varepsilon+\delta\big(F(x)+\varepsilon+||\alpha||_1\big).
    \]
Now we bound $F(x)-F(y)$. Fix some $\beta\in I$ such that 
\[
 \inProd{y}{\beta}\leq F(y)+\varepsilon.
\]
Note that 
    \[
    \inProd{x}{\beta}- \inProd{y}{\beta}\leq \sum_{x_i\neq 0}|x_i-y_i|\beta_i-\sum_{x_i=0}y_i\beta_i\leq \sum_{x_i\neq 0}\delta x_i\beta_i\leq \frac{\delta}{1-\delta}\sum_{i=1}^d y_i\beta_i\leq \frac{\delta}{1-\delta}F(y).
    \]
Hence, 
\[
F(x)\leq  \inProd{x}{\beta}\leq  \inProd{y}{\beta}+\frac{\delta}{1-\delta}F(y)\leq F(y)+\varepsilon+\frac{\delta}{1-\delta}\Big(F(x)+\delta\big(F(x)+\varepsilon+||\alpha||_1\big)\Big).
\]
    So taking $\delta>0$ such that $\frac{\delta}{1-\delta}\Big(F(x)+\delta\big(F(x)+\varepsilon+||\alpha||_1\big)\Big)<\varepsilon$ will imply that 
    \[
    |F(x)-F(y)|\leq 2\varepsilon.
    \]
\end{proof}

\begin{customthm}{\ref{t:sco}}[Binary classification vs. stochastic convex optimization]
Let $C\subseteq\{\pm 1\}^X$ be a binary concept class.
If for some $\beta<1/2$ and $\alpha > 0$ there exists an $(\alpha,\beta)$-reduction from the task of learning $C$ in the realizable case to a stochastic convex optimization task in $\mathbb{R}^d$, with loss functions $\{\ell_z\}_{z\in Z}$ satisfying $\ell_z(w)<\infty$ for all $z\in Z$, $w\in W$.	
Then
\[d\geq \max\{\vc(C), \vc^\star(C)-1\}.\]  
\end{customthm}
We are going to prove \Cref{t:sco} in a slightly more general setting, namely, we allow the learning task to have \emph{randomized hypotheses}, see \Cref{ex:pac-rand} for the definition. This is not a huge generalization, but it provides conformity with subsequent \Cref{ex-trivial-reduction}, which illustrates that the above statement trivializes if we allow $\beta\geq 1/2$.
\begin{proof}
 Let $r=(r_\inp,r_\outp)$ be an $(\alpha,\beta)$-reduction from $C$ to a stochastic convex optimization task with a convex loss function $L:W \times Z \to \mathbb{R}{\geq 0}$, where $W \subseteq \mathbb{R}^d$ is a convex set. To lower bound $d$ in terms of $\vc(C)$ and $\vc^\star(C)$, we can assume, without loss of generality, that $C$ is a finite class over a finite domain $X$. We achieve this by replacing $C$ with a finite subclass that has the same VC and dual VC dimensions. Consequently, we also assume $Z$ is finite by focusing on $r_\inp(X \times \{\pm1\})$.

   Let $P$ be the collection of $C$-realizable distributions over $ X\times \{\pm 1\}$, equipped with the total variation metric 
   \[
   \TV(D,D')=\sup_A |D(A)-D'(A)|.
   \]
where the supremum is over all measurable events $A$. We define the involution $D\to -D$ on the set of distributions over $Z$ as
\[-D(A)=D(\{(x,-y)\;:\; (x,y)\in A\}).\]
Note that $P$ is not, in general, closed under this involution. That is, $-D$ is not necessarily in $P$ even if $D$ is. We say that $S\subseteq P$ is an $n$-sphere if there is a homeomorphism  $\varphi:\mathbb{S}^{n}\to S$ such that $\varphi(-u)=-\varphi(u)$ for all $u\in \mathbb{S}^n$. We are now going to show that the existence of an $n$-sphere implies that $d\geq n+1$.

Let $S\subseteq P$ with a homeomorphism $\varphi$ be such $n$-sphere 
and define $G\subseteq \mathbb{S}^{n}\times W$ as
\[
G=\{(u,w)\;: L_{r_\inp\big(\varphi(u)\big)}(w)<\OPT_W\big(r_\inp(\varphi(u))\big)+\alpha/2\}.
\]
That is, for $u\in \mathbb{S}^n$, the set $G_u=\{w\in W\;:\; (u,w)\in G\} \subseteq W$
is the set of all $w\in W$ witnessing the $\alpha/2$-optimal loss with respect to $L_{r_\inp(\varphi(u))}$. In order to use Theorem $\ref{t:Borsuk-Ulam}$ we now need to show that $G_u=\{w\in W\;:\; (u,w)\in G\} \subseteq W$ is nonempty, that its convex hull is disjoint from the convex hull of $G_{-u}$, and that $G^w=\{u\in \mathbb{S}^n\;:\; (u,w)\in G\}$ is open. 

As $L_{r_\inp(\varphi(u))}$ is a convex function, it trivially follows that $G_u$ is nonempty and convex. We now want to check that for all $u\in \S^n$ the sets $G_u$ and $G_{-u}$ are disjoint.

Indeed, as $r$ is an $(\alpha,\beta)$-reduction, $r_\outp(w)$ is $\beta$-optimal for $\varphi(x)$ whenever $w$ is $\alpha$-optimal for $r_\inp(\varphi(u))$; in particular, $L^{01}_{\varphi(u)}\big(r_\outp(w)\big)\leq \beta$ for $w\in G_u$. Also, for any $D\in P$, and a randomized hypothesis $h\in [-1, 1]^X$ we have:

\begin{align*}
    L^{01}_{D}(h)+L^{01}_{-D}(h) &=
        \Ex_{x,y\sim D} \frac{1}{2} \left|h(x) - y\right| 
        + \Ex_{x,y\sim -D} \frac{1}{2} \left|h(x) - y\right| \\
        &=\Ex_{x,y\sim D} \frac{1}{2} \left( \left|h(x) - y\right| + \left|h(x) + y\right|  \right)
        =\Ex_{x,y\sim D} \frac{1}{2} \cdot 2 = 1. 
\end{align*}
So for any $w \in G_u$ it holds
\begin{align*}
    L^{01}_{\varphi(-u)}\big(r_\outp(w)\big) 
    = 1- L^{01}_{\varphi(u)}\big(r_\outp(w)\big) \geq 1 - \beta > 1/2 > \beta,
\end{align*}
and $w\notin G_{-u}$.

We now need to show that $G^w=\{u\in \mathbb{S}^n\;:\; (w,u)\in G\}$ is open. Define $p:\mathbb{S}^n\to \mathbb{R}^{|Z|}$ by $p(u)=\big(p_z(u)\big)_{z\in Z}$ where

\[
p_z(u)=\Pr_{(x,y)\sim \varphi(u)}(z=r_\inp(x,y)).
\]
Note that $p$ is continuous by properties of the total variation metric. Thus if we set $\alpha_w=\big(\ell_z(w)\big)_{z\in Z}\in \mathbb{R}^{|Z|}$ we get  
\begin{align*}
    &L_{r_\inp\big(\varphi(u)\big)}(w)=\sum_{z\in Z}p_z(u)\ell_z(w)=\inProd{p(u)}{\alpha_w},
    \\&\OPT_W\Big(r_\inp\big(\varphi(u)\big)\Big)=\inf_{w\in W} \inProd{p(u)}{\alpha_w}.
\end{align*}
Now, by Lemma \ref{l:Inf-Lin-Contin}, for all $w\in W$ the function $u\to L_{r_\inp(u)}(w)-\OPT_W(r_\inp(u))$ is continuous. Thus, $G^w$ is indeed open. 

It can now be seen that $G$ satisfies the condition of $\Cref{t:Borsuk-Ulam}$, and so $d\geq n+1$.
Finally, we refer the reader to the proof of Theorem~D in \cite{chase2024local} for the following statement regarding the existence of spheres:
\begin{prop}\label{prop-spheres}
    For a class $C\subseteq \{\pm1\}^X$, there is an $n$-sphere in the space of realizable distributions of $C$ whenever $n\leq \max\{\vc(C) - 1, \vc^*(C)-2\}$.
\end{prop}

Combining \Cref{prop-spheres} with $d\geq n+1$, we get $g\geq \max\{\vc(C) - 1, \vc^*(C)-2\} + 1 = \max\{\vc(C), \vc^*(C)-1\}$, as needed.
\end{proof}

\begin{customthm}{\ref{t:sco2}}
    Let $C\subseteq\{\pm 1\}^X$ be a binary concept class.
Assume that there exists an \underline{exact} reduction from the task of learning $C$ in the realizable case with randomized hypotheses to a stochastic convex optimization task in $\mathbb{R}^d$ with $\infty$-valued loss functions $\ell_z:W\to \mathbb{R}_{\geq0}\cup\{\infty\}$.
Then, 
\[d\geq \max\{\vc(C)-1, \vc^*(C)-1\}.\]  
\end{customthm}
\begin{proof}
    By Theorem \ref{t:halfspaces-complete} (see the proof in \Cref{sec-proofs-repr}) the sign-rank of $C$ is at most $d+1$. By \cite{Alon17}, both $\vc(C)$ and $\vc^*(C)$ lower-bound the sign-rank from which the statement of the theorem follows.
\end{proof}

\subsection{Borsuk-Ulam for closed convex relations}
\begin{lem}\label{l:ConvexLemma}
    Let $G$ be as in Theorem \ref{t:Borsuk-Ulam} and let us define $A_{x,\delta}$, for $x\in \mathbb{S}^d$ and $\delta>0$, as a convex hull of the set $\{(x',w)\in G\;: d(x,x')<\delta\}$. Then for any $\varepsilon>0$ there exist $\delta>0$ such that for any $x\in \mathbb{S}^d$ and $x'\in A_{x,\delta}$, $x'$ is $\varepsilon$-close to $G$. 
\end{lem}
\begin{proof}
    By Carath\'eodory's theorem, it is enough to show that any convex combination of $d+1$ points from $\{(x',w)\in G\;: d(x,x')<\delta\}$ is at most $\varepsilon$ away from $G$.
     Let $\Lambda=\{\lambda\in [0,1]^{d+1}\;:\sum_{i=1}^d \lambda_i=1\}$ and define the function
     $F: G^d\times \Lambda\to \mathbb{R}$ by 
 \[
F(g,\lambda)=d\left(G, \sum_{i=1}^{d+1} \lambda_i g_i\right).
 \]

By construction, $F$ is continuous and, as $G^d\times \Lambda$ is compact, it is uniformly continuous. Now, for  $g=(g_1,g_2,\dots g_{d+1})\in G^{d+1}$, where $g_i=(x_i,w_i)$, note that if for some $x'\in\S^d$, all $x_i=x'$, then $w_i\in G_{x'}$ for all $i$ and, by convexity of $G_{x'}$, $\sum_{i=1}^d \lambda_i w_i\in G_{x^\star}$ for all $\lambda\in \Lambda$. Hence in this case $F(g,\lambda)=0$ for all $\lambda$. Thus, by the uniform continuity of $F$, for any $\varepsilon>0$ there is some $\delta>0$ such that if $x_1,x_2\dots, x_{d+1}$ are in a $\delta$-ball around $x'\in \S^d$ then $F(x,\lambda)<\varepsilon$ for all $\lambda\in \Lambda$, which concludes the proof.
\end{proof}

\begin{customthm}{\ref{t:Borsuk-Ulam}}[Borsuk-Ulam for closed convex relations]
Let $W$ be a compact convex set in $\mathbb{R}^k$ and let $G$ be a closed set inside $\mathbb{S}^d \times W$, where $\mathbb{S}^d$ is a $d$-dimensional unit sphere. Additionally, suppose that:
\begin{itemize}
    \item For any $x\in \mathbb{S}^d$, the set $G_x = \{w\in W~|~(x,w)\in G\}$ is nonempty and convex;
    \item For any $x\in \mathbb{S}^d$, the sets $G_x$ and $G_{-x}$ are disjoint.
\end{itemize}	
Then $k\geq d+1$.
\end{customthm} 
\begin{proof}
Assume towards contradiction that we have such $G$ with $d\leq k$. Let $-G=\{(-x,w)\;:\;(x,w)\in G\}$. Note that as $G_x\cap G_{-x}=\emptyset$ for all $x$, we have $-G\cap G=\emptyset$ and, since both are closed, $d(G,-G)>0$. That is, there is some $\varepsilon>0$ such that $d(g,g')>\varepsilon$ for all $g\in G$, $g'\in -G$.

Let $\delta>0$ be a small constant to be chosen later. For each $x\in \mathbb{S}^d$ let $B_\delta(x)$ be the ball with radius $\delta$ centered at $x$. By compactness we have some finite $T\subset \mathbb{S}^d$ such that $\{B_\delta(t)\}_{t\in T}$ is an open cover of $\mathbb{S}^d$. Similarly to the proof of \Cref{t:BU-open}, let $\{\rho_t\}_{t\in T}$ be a partition of unity subordinate to this cover.

Now for each $t\in T$ choose some $w_t\in G_t$ in an arbitrary way, and define $\chi:\S^d\to \R^{d+1}$, $\phi:\S^d\to W$, and $\Phi = (\chi, \phi):\S^d\to \R^{d+1}\times W$ as 
\begin{align*}
    \chi(x)&=\sum_{t\in T}\rho_t(x)t, \\
    \phi(x)&=\sum_{t\in T}\rho_t(x)w_t.
\end{align*}
Note that in order to define $\chi(x)$ we assume that the unit sphere $\S^d$ is canonically embedded into $\R^{d+1}$.

As $\rho_t(x)= 0$ whenever $d(x,t)>\delta$, $\chi(x)$ is a convex combination of points inside $B_\delta(x)$. As the latter is convex, $d\left(\Phi(x), (x, \phi(x))\right) = d\left(\chi(x), x\right) \leq \delta$ for all $x\in \S^d$. Also, $\Phi(x)$ is in the convex hull of $\{(x',w)\in G\;:\; d(x,x')\leq \delta\}$, so, by \Cref{l:ConvexLemma}, for
$\delta>0$ small enough, $\Phi(x)$ is $(\varepsilon/4)$-close to $G$. Additionally, assuming $\delta \leq \varepsilon/4$, for any $x\in \S^d$ we get
\[
d\big(G,(x,\phi(x))\big)\leq d\big(G,\Phi(x)\big)+
d\big(\Phi(x),(x,\phi(x))\big)\leq \frac{\varepsilon}{2}.
\]
By the definition of $-G$, $d\big(G,(-x,\phi(x))\big) \leq \varepsilon/2$. As $x$ is arbitrary from $\S^d$, by changing $x$ to $-x$ we get $d\big(G,(x,\phi(-x))\big) \leq \varepsilon/2$.

Now, since $\phi$ is continuous, $d\leq k$ implies, by the
Borsuk-Ulam theorem, that there is some $x\in \mathbb{S}^d$ such that $\phi(x)=\phi(-x)$. For such $x$ we have
\begin{align*}
    d(G,-G) &\leq d\big(G,(x,\phi(x))\big)
                +d\big(-G, (x,\phi(x))\big) \\
        &= d\big(G,(x,\phi(x))\big)
                +d\big(-G, (x,\phi(-x))\big)
        \leq \varepsilon,
\end{align*}
which contradicts the fact that $d(G,-G)>\varepsilon$.
\end{proof}

\section{Proofs for geometric representations}\label{sec-proofs-repr}

The statement of \Cref{prop-repr-is-red} below is made more formal than in the main part where it was originally formulated.
\begin{customprop}{\ref{prop-repr-is-red}}[Representations are reductions]
	Let $\T_1=(H_1,C_1,Z_1,P_1)$ be a realizable PAC-learning task over domain $X_1$, that is, $H_1 = \{\pm1\}^{X_1}$, $C_1\subseteq \{\pm 1\}^{X_2}$, $Z_1 = X_1\times \{\pm1\}$, and $P_1$ be the family of all distributions on $Z_1$, realizable by $C_1$.
	
	Let $r\colon X_1 \rightarrow X_2$ be an $\alpha$-representation of $C_1$ by a class $C_1\subseteq \{\pm 1\}^{X_2}$. Let $\T_2=(H_2,C_2,Z_2,P_2)$ be a PAC-learning task over $X_2$, where $P_2 \supseteq r(P_1)$, that is, $P_2$ contains the images of all distributions in $P_1$ under $r$.
	
	Let us define $r_\inp(x,y) = (r(x), y)$, for $x\in X_1$ and $y\in \{\pm1\}$, and $r_\outp(h)(x) = h(r(x))$, for $x\in X_1$ and $h\in H_2$. Then the pair of maps $r_\inp$, $r_\outp$ is a $(\gamma, \gamma + \alpha)$-reduction for any $\gamma\geq 0$.
\end{customprop}
\begin{proof}
	Let $D_1\in P_1$ and $D_2=r_\inp(D_1)$. By the condition on $P_2$, $D_2\in P_2$. Suppose $h_2\in H_2$ is $\gamma$-optimal for $D_2$, that is, $L_{D_2}(h_2)\leq \OPT_{C_2}(D_2) + \gamma$.
	
    As $\T_1$ is realizable, so is $D_1$. Hence, by the definition of $\alpha$-representation, $\OPT_{C_2}(D_2)\leq \alpha$, and so $L_{D_2}(h_2) \leq \alpha + \gamma$. Let $h_1 = r_\outp(h_2)$. Then 
	\begin{align*}
		L_{D_1}(h_1) = \Pr_{x,y\sim D_1}\left[h_1(x) \neq y \right]
			= \Pr_{x,y\sim D_1}\left[h_2(r(x)) \neq y \right]
			= \Pr_{x,y\sim D_2}\left[h_2(x) \neq y \right]
			= L_{D_2}(h_2).
	\end{align*}
	Thus, $L_{D_1}(h_1) = L_{D_2}(h_2) \leq \alpha + \gamma = \alpha + \gamma$. As $D_1$ is realizable, $\OPT_{C_1}(D_1) =0$, and so $h_1$ is $(\alpha + \gamma)$-optimal, as needed.
\end{proof}

\begin{customthm}{\ref{t:halfspaces-complete}}[Half-spaces are complete for exact reductions to SCO]
Let $C\subseteq\{\pm 1\}^X$ be a finite concept class.
If for some $\beta<1/2$ and $\alpha$ there exists an exact $(\alpha,\beta)$-reduction from the task of learning $H$ in the realizable case to a stochastic convex optimization task in $\mathbb{R}^d$, then $C$ has an exact representation by homogenous half-spaces in $\mathbb{R}^{d+1}$. In other words, the sign-rank of~$C$ is at most~$d+1$.
\end{customthm}
\begin{proof}
    Without losing generality, we can assume that for every $x,y\in X\times \{\pm1\}$ there is some $c\in C$ such that $c(x)=y$; otherwise, we can restrict the domain to those $x$'s for which this condition is satisfied, and the restricted class will be trivially equivalent to the original one. Recall that in the setup of an SCO task in $\R^d$ we assume that the loss functions are defined on a convex set $W \subseteq \R^d$. We will further assume that this $\R^d$ is embedded into $\R^{d+1}$ as $\R^d = \{\ol{x}~|~x_{d+1}=1\}$. The reason for this is that now every affine hyperplane in $\R^d$ can be uniquely extended to a homogenous hyperplane in $\R_{d+1}$.
    
    Fix some $c\in C$ and look at the following game: the player chooses some $w\in W$ and the adversary chooses some $x\in X$, then the player suffers a loss of $L_{z}(w)$ for $z=r_\inp\big(x,c(x)\big)$. Strategies for the adversary are distributions over $X$, which are equivalent to distributions over $X\times \{\pm 1\}$, realizable by $c$. Since $r$ is exact, for each $c$-realizable distribution $D$ there is some $w\in W$ such that $L_{r_\inp(D)}(w)=0$. Hence, by the minimax theorem, there is some $w_c\in W$ such that $L_{r_\inp(D)}(w_c)=0$ for all such $D$. Note that we are using the fact that the class is finite for the applicability of the minimax. 

    Now for each $x, y\in X \times \{\pm 1\}$ define $V_{x,y}$ as a convex hull of the set $\{w_c:c\in C, c(x)=y\}$. By construction, $V_{x,y}$ is convex and compact, also, as we assumed that there is $c\in C$ with $c(x)=y$, it is nonempty. Moreover, as $L_{r_\inp(x,y)}$ is convex, nonnegative, and $0$ on all $w_c$ for $c\in C$ such that $c(x)=y$, it follows that $L_{r_\inp(x,y)}(w)=0$ for $w\in V_{x,y}$.

    As $r$ is an $(\alpha,\beta)$-reduction, for any $w\in V_{x,y}$ the $0/1$ loss of $r_\outp(w)$ on the distribution which is supported only on $(x,y)$  is  at most $\beta<\frac{1}{2}$. Thus $r_\outp(w)(x)=y$, which implies that $V_{x,1}\cap V_{x,-1}=\emptyset$. As both are convex and compact, by the Hahn-Banach separation theorem, there is an affine hyperplane in $\R^d$ strictly separating them. Recall that, by the choice of an embedding of $\R^d$ into $\R^{d+1}$, we can extend this hyperplane to a homogenous hyperplane in $\R^{d+1}$. That is, for every $x\in X$, there is $\varphi_x\in \R^{d+1}$ such that $\sign (\langle \varphi_x,w\rangle)$ is $y$ for $w\in V_{x,y}$. 
    
    With this, the maps $\varphi:X\to \R^{d+1}$ and $w\colon C\rightarrow \R^{d+1}$ defined as $\varphi(x) = \varphi_x$ and $w(c) = w_c$ satisfy $\sign(\langle w(c), \varphi(x) \rangle) = c(x)$ for all $c\in C$ and $x\in X$, witnessing that the sign-rank of~$C$ is at most~$d+1$.
\end{proof}

\begin{customthm}{\ref{t:approxVSexact}}[Exact vs approximate sign-rank]
For every integer $d\geq 0$, there exists a finite concept class $C_d$ whose $(1/3)$-sign-rank is at most $d$, and the sign-rank is at least~$d^{\Omega(\log d)}$.
\end{customthm}

\begin{proof}
    Let $H_d$ be the class of half-spaces in $\mathbb{R}^d$ and let $C_d=\{\maj(h_1,h_2,h_3)\;:\; h_1,h_2,h_3\in H_d\}$ be the class of majority vote of three such half-spaces. It is easy to see that an identity map on $\R^d$ is an $(1/3)$-representation of $C_d$ by $H_d$. Indeed, let $D$ be a distribution realizable by $C_d$ and let $c=\maj(h_1,h_2,h_3)$ be an element in $C_d$ with $L_D(c)=0$. For $i\in \{1,2,3\}$, let $A_i$ be the event that $h_i(x)=y$ for a random pair $(x,y) \sim D$. By definition of $c$, for any $x$ at least two of the events $A_1,A_2,A_3$ will occur, hence 
    \[
      \Ex \Big(1_{A_1}+1_{A_2}+1_{A_3}\Big)\geq 2.
    \]
    This trivially implies that $\Ex \left(1_{A_i}\right)\geq \frac{2}{3}$ for some $i\in \{1,2,3\}$. Equivalently,
    $L_D(h_i)\leq 1/3$, so the identity map is indeed a $(1/3)$-representation. 

    For the second part, in Corollary~1.2 in~\citet{bun21} the authors prove that a class of $2$-intersections of certain signs of weighted majorities on a $4m^2$-dimensional Boolean hypercube has sign-rank $m^{\Omega(\log m)}$. We note that that some details are not stated in Corollary~1.2 explicitly, but can be easily extracted from the proof of their Theorem~1.1. By embedding the hypercube into $\R^{4m^2}$ in a standard way, and noting that a sign of a weighted majority on a hypercube can be expressed by a sign of a homogenous hyperplane, we get that for any $d$, the sign-rank of the class $I_d$ of intersections of two half-spaces in $\R^d$ is at least $d^{\Omega(\log d)}$. Note that $\Omega$ hides coefficients arising from going from $4m^2$ to an arbitrary $d$.
    
    Finally, it is easy to see that the class $I_d$ can be embedded into $C_{d+1}$, yielding the desired lower bound on the sign-rank of $C_d$
\end{proof}

\section{Several probabilistic sign-ranks and proofs of Theorems~\ref{t-prob-sr-margin} and~\ref{t-det-sr-margin} }\label{sec-prob-sign-ranks}
Before proving \Cref{t-prob-sr-margin}, let us elaborate on several notions of probabilistic sign-ranks related to our work. We start with \emph{probabilistic dimension complexity} see Definition~2 in~\cite{KamathMS20}. Note that the authors consider it for different families of loss functions, but the definition below is specifically for $0/1$-loss. We also give all definitions for \emph{partial} concept classes and slightly align the notation in line with ours. 
\begin{definition}[Probabilistic dimension complexity,~\cite{KamathMS20}]\label{def-dc-kamath}
	For a (partial) class $C$ over domain $X$ and ${\delta>0}$, a \emph{probabilistic $\delta$-dimension complexity} $\dc_C(\delta)$ of $C$ is the smallest $d$ for which there is a distribution $R$ over maps $r\colon X\rightarrow \R^d$ such that for all distributions $D$ over $X\times \{\pm1\}$, realizable by $C$, it holds
	\begin{align*} 
		\Ex_{r\sim R} \left[ \inf_{w\in \R^d} 
			\Pr_{x, y\sim D} \bigl[\sign\langle w, r(x)\rangle \neq y\bigr] \right] \leq \delta.
	\end{align*}
\end{definition}

By expanding the intermediate definitions, one can see that this is very close to how we defined the sign-rank$^*$:
\begin{customdef}{\ref{def-sr-prob-star}}[Probabilistic sign-rank$^*$, restated]
	For a (partial) class $C$ over domain $X$, $\delta>0$, and $\alpha < 1/2$, a \emph{$(\delta, \alpha)$-sign-rank$^*$} $\sr^*_C(\delta, \alpha)$ of $C$ is the smallest $d$ for which there is a distribution $R$ over maps $r\colon X\rightarrow \R^d$ such that for all distributions $D$ over $X\times \{\pm1\}$, realizable by $C$, it holds
	\begin{align*} 
		\Pr_{r\sim R} \left[ \inf_{w\in \R^d} 
			\Pr_{x, y\sim D} \bigl[\sign\langle w, r(x)\rangle \neq y\bigr]  > \alpha \right] \leq \delta.
	\end{align*}
\end{customdef}

The following mutual bounds between the two are rather trivial:
\begin{prop}\label{prop-kamath-and-us}
	For a (partial) class $C$, $\delta>0$, and $\alpha < 1/2$, it holds:
	\begin{align*}
		\sr^*(\delta/\alpha, \alpha) &\leq \dc(\delta), \\
		\dc(\delta + \alpha(1-\delta))  &\leq \sr^*(\delta, \alpha),
	\end{align*}
	where, for compliance with the definitions, in the first bound we additionally assume that $\delta< \alpha/2$.
\end{prop}
Note that here and below, we drop the underscripts identifying the class in $\sr$, $\dc$, etc., whenever the class is clear from the context.
\begin{proof}
	Note that the condition on the distribution $R$ witnessing the respective $d$ in Definitions~\ref{def-dc-kamath} and~\ref{def-sr-prob-star} is stated as for all $D$, $\Ex_{r\sim R}\left[F_D(r)\right]\leq \delta$ for $\dc(\delta)$ and $\Pr_{r\sim R}\left[F_D(r)>\alpha\right]\leq \delta$ for $\sr^*(\delta, \alpha)$ respectively. Here, informally, $F_D(r)$ is a fit of the distribution $r(D)$ to the class of half-spaces in $\R^d$, however, we only need that $F_D(r)\in [0,1]$.	
	By Markov's inequality, $\Pr\left[F_D(r)>\alpha\right]\leq \delta/\alpha$ whenever  
	$\Ex\left[F_D(r)>\alpha\right]\leq \delta$, and so the same $R$ witnessing $\dc(\delta) \leq d$ also witnesses $\sr^*(\delta/\alpha,\alpha) \leq d$, yielding the first inequality. 
	
	Note that Markov's inequality comes from the fact that the function $F_D$ maximizing $\Pr\left[F_D(r)>\alpha\right]$ provided $\Ex\left[F_D(r)\right]\leq\delta$ is $F_D = \alpha + 0$ w.p. $\delta/\alpha - 0$ and $F_D = 0$ otherwise. In the same spirit, the function that maximizes $\Ex\left[F_D(r)\right]$ provided $\Pr\left[F_D(r)>\alpha\right]\leq \delta$ is $F_D = 1$ w.p. $\delta$ and $F_D = \alpha$ otherwise. From this, $\Ex\left[F_D(r)\right] \leq \delta + \alpha(1-\delta)$ whenever $\Pr\left[F_D(r)>\alpha\right]\leq \delta$, yielding, in a similar way, the second inequality.
\end{proof}
 
It is natural to compare the above two definitions to a canonical notion of \emph{probabilistic sign-rank} from communication complexity. Its formulation below is from Definition~23 in~\cite{KamathMS20}, where it is called \emph{point-wise probabilistic dimension complexity}. \begin{definition}[Probabilistic sign-rank]\label{def-sr-kamath}
	For a (partial) class $C$ over domain $X$ and ${\delta>0}$, a \emph{probabilistic $\delta$-sign-rank} $\sr_C(\delta)$ of $C$ is the smallest $d$ for which there is a distribution $R$ over pairs of maps $(r\colon X\rightarrow \R^d, \omega\colon C\rightarrow \R^d)$ such that for all $h\in C$ and $x\in \sup(h)$, it holds
	\begin{align*} 
		\Ex_{r, \omega \sim R} \bigl[\sign\langle \omega(h), r(x)\rangle \neq h(x)\bigr] \leq \delta.
	\end{align*}
\end{definition}

Finally, let us define the following two relaxations of sign-rank, that we will use in the proof of \Cref{t-prob-sr-margin}. We will not give them descriptive names, as they are purely technical and used only to give a uniform proof for all of the three versions of sign-rank above.

\begin{definition}
	For a (partial) class $C$ over domain $X$, $\delta>0$, and $\alpha < 1/2$, let $\sr_C^\dag(\delta)$ and $\sr_C^\dag(\delta, \alpha)$ be the smallest $d$ for which there is a distribution $R$ over pairs of maps $(r\colon X\rightarrow \R^d, \omega\colon C\rightarrow \R^d)$ such that for all $h\in C$ and all distributions $D$, realizable by $h$, it holds
		\begin{align*} 
			\Ex_{r, \omega \sim R} \left[
				\Pr_{x, y\sim D} \bigl[\sign\langle w(h), r(x)\rangle \neq y\bigr] \right] \leq \delta
		\end{align*}
		for $\sr_C^\dag(\delta)$ and 
		\begin{align*} 
			\Pr_{r, \omega \sim R} \left[
				\Pr_{x, y\sim D} \bigl[\sign\langle w(h), r(x)\rangle \neq y\bigr] \geq \alpha \right] \leq \delta
		\end{align*}
		for $\sr_C^\dag(\delta, \alpha)$.
\end{definition}

\begin{prop}\label{prop-sr-dag}
	For a (partial) class $C$ over domain $X$, $\delta>0$, and $\alpha < 1/2$, it holds
	\begin{align*}
		\sr^\dag(\delta/\alpha, \alpha) &\leq \sr^\dag(\delta), \\
		\sr^\dag(\delta + \alpha(1-\delta))  &\leq \sr^\dag(\delta, \alpha),\\
		\dc(\delta) \leq \sr(\delta)&\leq \sr^\dag(\delta),  \textrm{and}\\
		\sr^*(\delta, \alpha)  &\leq \sr^\dag(\delta, \alpha).
	\end{align*}
	where in the second bound we assume $\delta < \alpha/2$.
\end{prop}
\begin{proof}
	Note that the relation between $\sr^\dag(\delta)$ and $\sr^\dag(\delta, \alpha)$ is the same as between $\dc(\delta)$ and $\sr^*(\delta, \alpha)$, so the proof of first two inequalities is similar to \Cref{prop-kamath-and-us}.
	The fact that $\dc(\delta) \leq \sr(\delta)$ is by Proposition~24 in~\cite{KamathMS20}. For $\sr(\delta) \leq \sr^\dag(\delta)$, note that for $R$ witnessing $\sr^\dag(\delta)\leq d$, it holds 
		\begin{align*} 
			\sup_{h\in C} &\sup_{x\in \sup(h)} \Ex_{r, \omega \sim R} \bigl[\sign\langle \omega(h), r(x)\rangle \neq h(x)\bigr] \\
			&= \sup_{h\in C} \sup_{\chi_{x,y}\ll h} \Ex_{r, \omega \sim R} \left[
					\Pr_{x, y\sim D} \bigl[\sign\langle w(h), r(x)\rangle \neq y\bigr] \right] \\
			&\leq \sup_{h\in C} \sup_{D \ll h}\Ex_{r, \omega \sim R} \left[
				\Pr_{x, y\sim D} \bigl[\sign\langle w(h), r(x)\rangle \neq y\bigr] \right] \leq \delta,
		\end{align*}
	where $\chi_{x,y}$ is a one-point distribution of the example $(x, y)$, and $D\ll h$ denotes that $D$ is realizable by $h$. Thus, $R$ also witnesses $\sr(\delta) \leq d$ and hence $\sr(\delta) \leq \sr^\dag(\delta)$.
	
Finally, for $R$ witnessing $\sr^\dag(\delta)\leq d$, the same $R$ restricted to the first coordinate trivially witnesses $\sr^*(\delta)\leq d$, and so $\sr^*(\delta, \alpha)  \leq \sr^\dag(\delta, \alpha)$.
\end{proof}

We are now ready to prove \Cref{t-prob-sr-margin}.
\begin{customthm}{\ref{t-prob-sr-margin}}[Probabilistic sign-rank of halfspaces with margin]
Let $C_{n}$ be the partial class of linear classifiers with constant margin $\gamma=1/3$ on the $n$-dimensional sphere $\mathbb{S}^n$. Then for $\alpha \in (0, 1/2)$ and $\delta\in (0,1)$, the $(\delta, \alpha)$-sign-rank$^*$ $d=\sr^*(\delta, \alpha)$ of $C_{n}$ is at most
    \[d= O\Bigl(\log\frac{1}{\alpha\delta}\Bigr).\]
The randomized $(\delta, \alpha)$-representation witnessing it is linear, that is, the respective distribution is over linear maps $\S^n\rightarrow \R^d$, where we treat $\S^n$ as a unit sphere in~$\R^{n+1}$.

Moreover, probabilistic $\delta$-dimension complexity $\dc(\delta)$ and probabilistic $\delta$-sign-rank $\sr(\delta)$ are at most $O(\log(1/\delta))$.
\end{customthm}
\begin{proof}
We will prove that for given $\alpha$ and $\delta$, $\sr^\dag(\delta\cdot\alpha) \leq O(\log 1/\alpha\delta)$. All the claimed bound then follow from the $\sr^*(\delta)\leq \sr^\dag(\delta)$ and $\dc(\delta) \leq \sr(\delta)\leq \sr^\dag(\delta)$ bounds from \Cref{prop-sr-dag}.

Let $d>0$, to be chosen later. Note that in this case the domain and the class are both equal to $\S^n$, and we need to construct a distribution $R$ over pairs of maps $(r\colon \S^n\rightarrow \R^d, \omega \colon \S^n\rightarrow \R^d)$. In fact, we put both maps to be equal to the same random linear map $\mathbb{R}^{n+1}\to \mathbb{R}^d$ whose entries $R_{i,j}$ are independent normal distributions $\NN(0,1)$. We refer to Corollary~20 in \cite{Ben-DavidES02} in the case of a single half-space for the following statement:
\emph{For any  half-space $w$ in $\R^{n+1}$ with margin $\gamma$, and $x\in \mathbb{S}^n$ in its support, so $|\inProd{w}{x}|>\gamma$, it holds}
\[
\Pr_{r\sim R}[\sign\big(\inProd{r(w)}{r(x)}\big)\neq \sign\big(\inProd{w}{x}\big)]\leq 4e^{\frac{-d\gamma^2}{8}}.
\]
We note that the above statement is a close relative of the Johnson-Lindenstrauss lemma, where this particular construction of the random projection is from~\cite{arriaga06}. In particular, as is usual for this lemma, this estimate does not depend on $n$. With this we get that for any distribution $D$ realizable by the half-space with margin $w$ we have


\[
\Ex_{r\sim R}\Big[\Pr_{(x,y)\sim D}\big[\sign\big(\inProd{r(w)}{r(x)}\big)\neq y\big] \Big]=\Pr_{(x,y)\sim D}\Big[\Pr_{r\sim R}\big[\sign\big(\inProd{r(w)}{r(x)}\big)\neq \sign\big(\inProd{w}{x}\big)\big] \Big]\leq 4e^{\frac{-d\gamma^2}{8}}.
\]

But, by recalling that $\gamma = 1/3$ is a fixed constant, the above will be less then $\delta\cdot\alpha$ for $d=\frac{10}{\gamma^2}\log\frac{1}{\alpha\delta}=O\left(\log(1/\alpha\delta)\right)$, as needed.
\end{proof}

We will now go for the proof of \Cref{t-det-sr-margin}, starting with the following lemma.
\begin{lem}\label{l:HellyNum}
    For $d\in \N$ and $0<\alpha<1/(d+1)$, an $\alpha$-representation of a (partial) concept class $C$ by half-spaces in $\mathbb{R}^d$ is an exact representation. In particular, if the $\alpha$-sign-rank of $C$ is at most $d = (1-\alpha)/\alpha$, then the sign-rank of $C$ is at most $d$.
\end{lem}
\begin{proof}
     Let $r: X\to \mathbb{R}^d$ be such $\alpha$-representation, and for any $D$ distribution over $X\times \{\pm 1\}$ let $L_D:\mathbb{R}^d\to [0,1]$ be the induced loss function, i.e. \[
     L_D(w)=\Pr_{(x,y)\sim D}\big(\sign(\inProd{w}{r(x)})\neq y\big)
     \]. For any $S\subseteq X\times\{\pm 1\}$, let $U_S$ denote the uniform measure on $S$, and
     define the set $V_S\subset \mathbb{R}^d$ by 
\[
V_S=\bigcap_{(x,y)\in S} \{w\in \mathbb{R}^d\;:\; \sign\big(\inProd{w}{r(x)}\big)=y\}.
\]
Note that $V_s$ is a convex set as an intersection of half-spaces. 

Now, let $S\subseteq X\times \{\pm 1\}$ be a finite sample realizable by $C$. As $r$ is an $\alpha$-representation, for any $T\subseteq S$ of size at most $d+1$ there is $w_T\in \mathbb{R}^d$ such that 
     \[
     \frac{1}{|T|}\bigl|\{(x,y)\in T\; :\sign\langle w_T,r(x) \rangle\neq y\}\bigr|=L_{U_T}({w_T})<\alpha<\frac{1}{|T|}.
     \]
Hence, $\sign\langle w_T,r(x)\rangle=y$ for all $(x,y)\in T$, and so $w_T\in V_T$. Thus, any intersection of at most $(d+1)$ sets $V_{\{(x,y)\}}$ for $(x,y)\in S$ is non-empty and, by Helly's theorem, the overall intersection $V_S$ is non-empty. Therefore, for each finite $S$ there exists $w_S\in \mathbb{R}^d$ such that $\sign\langle w_S,r(x)\rangle=y$ for all $(x,y)\in S$.

Now the class of half-spaces in $\mathbb{R}^d$ is a learnable class with VC-dimension of $d$, hence it satisfies the uniform convergence principle and by \Cref{t:UnifromConvergence} for any $\varepsilon>0$ and distribution $D$ realizable by $C$, we have that for all $n>0$

\[\Ex_{S\sim D^n} L_D(w_S)=
\Ex_{S\sim D^n}\Big(|L_{D}({w_S})-L_{S}({w_S})|\Big)\leq \Ex_{S\sim D^n}\Big(\sup_{w\in \mathbb{R}^d}|L_{D}(w)-L_{S}(w)|\Big)=O\bigl(\sqrt{d/n}\bigr).
\]
Hence, for any $n>0$ there is some $w\in \mathbb{R}^d$ such that $L_{D}(w)=O\bigl(\sqrt{d/n}\bigr)$, from which we deduce that $\OPT(D)=0$ and $r$ is an exact representation.
\end{proof}

\begin{customthm}{\ref{t-det-sr-margin}}[Deterministic sign-rank of halfspaces with margin]
Let $C_{n}$ be the partial class of linear classifiers with constant margin $\gamma=1/3$ on the $n$-dimensional sphere $\mathbb{S}^n$. Then for $\alpha \in (0, 1/2)$ the (deterministic) $\alpha$-sign-rank $d$ of $C_n$ is at least
    \[d \geq \min\Bigl\{\frac{1-\alpha}{\alpha},n+1\Bigr\}.\]
Moreover, if the respective $\alpha$-representation is continuous, then $d \geq n+1$.
\end{customthm}
\begin{proof}
Let $r:\S^{n}\to \R^d$ be an $\alpha$-representation of $C_n$ with $d<\frac{1-\alpha}{\alpha}$. By \Cref{l:HellyNum}, $r$ is an exact representation, so $d$ is at least the sign-rank of $C_n$. However, by Theorem~1.5 in~\cite{HatamiHM22}, the sign-rank of $C_n$ is exactly $n+1$.

Let us now assume that $\alpha\in (0, 1/2)$ and that the $\alpha$-representation $r:\S^{n}\to \R^d$ is continuous. Towards contradiction, suppose $d\leq n$. Then, by the Borsuk-Ulam theorem, there is some $x\in \S^{n}$ such that $r(x)=r(-x)$.  Let $D$ be the uniform distribution on $(x,1)$ and $(-x,-1)$, which is clearly realizable by $C$.
At the same time, trivially, for any $w\in \R^d$, $L_{r(D)}(w)=1/2$ and so $r(D)$ is not $\alpha$-realizable for $\alpha<1/2$.
\end{proof}

\bibliographystyle{plainnat}
\bibliography{bib}


\end{document}